\documentclass{article} 
\usepackage{iclr2026_conference,times}
\usepackage{graphicx}
\usepackage{marvosym}

\usepackage{mathtools,amsfonts,amssymb,bm}
\mathtoolsset{showonlyrefs}
\usepackage{graphicx}
\usepackage{booktabs}
\usepackage{multirow}
\usepackage{makecell}
\usepackage{adjustbox}
\usepackage{nicefrac}
\usepackage{xspace}
\usepackage{pifont}
\usepackage{amsthm}
\usepackage{algorithm}
\usepackage{algpseudocode}

\newtheorem{theorem}{Theorem}[section]

\newtheorem{proposition}[theorem]{Proposition}

\theoremstyle{definition}

\def\eqref#1{equation~\refeq{#1}}

\def\1{\bm{1}}

\DeclareMathAlphabet{\mathsfit}{\encodingdefault}{\sfdefault}{m}{sl}
\SetMathAlphabet{\mathsfit}{bold}{\encodingdefault}{\sfdefault}{bx}{n}

\usepackage{xcolor}
\usepackage{colortbl}
\usepackage{wrapfig}
\usepackage{booktabs}
\usepackage{graphicx}
\usepackage{amsmath}
\usepackage{bm}
\usepackage{amssymb}
\usepackage{tabularx}
\usepackage{graphicx}
\usepackage[hidelinks]{hyperref}
\usepackage{url}
\hypersetup{
  colorlinks=true,
  linkcolor={[rgb]{0.10,0.24,0.48}},
  citecolor={[rgb]{0.10,0.24,0.48}},
  urlcolor={[rgb]{0.10,0.24,0.48}},
  pdfborder={0 0 0}
}

\title{Evaluation Is All You Need for Multi-Modal Autonomous Driving}
\makeatletter
\let\iclrstandardauthor\author
\newcommand{\iclrmaybeauthorsup}[1]{%
  \if\relax\detokenize{#1}\relax
  \else
    \textsuperscript{#1}%
  \fi
}
\newcommand{\iclr@authorlist}{}
\newcommand{\iclr@affiliationlist}{}
\newcommand{\iclr@contributionlist}{}
\newcommand{\iclrAuthorBlockWidth}{0.96\textwidth}
\newcommand{\iclrTitleAuthorGap}{1.2ex}
\newcommand{\iclrAuthorAffiliationGap}{0.9ex}
\newcommand{\iclrAffiliationContributionGap}{0.7ex}
\newcommand{\iclrAuthorLineSpread}{1.12}
\newif\ificlr@firstauthor
\newif\ificlr@firstaffiliation
\newif\ificlr@firstcontribution
\iclr@firstauthortrue
\iclr@firstaffiliationtrue
\iclr@firstcontributiontrue
\renewcommand{\author}[2][]{%
  \ificlr@firstauthor
    \global\iclr@firstauthorfalse
  \else
    \g@addto@macro\iclr@authorlist{, }%
  \fi
  \g@addto@macro\iclr@authorlist{#2\iclrmaybeauthorsup{#1}}%
}
\newcommand{\affiliation}[2][]{%
  \ificlr@firstaffiliation
    \global\iclr@firstaffiliationfalse
  \else
    \g@addto@macro\iclr@affiliationlist{\\}%
  \fi
  \g@addto@macro\iclr@affiliationlist{\iclrmaybeauthorsup{#1}#2}%
}
\newcommand{\contribution}[2][]{%
  \ificlr@firstcontribution
    \global\iclr@firstcontributionfalse
  \else
    \g@addto@macro\iclr@contributionlist{\qquad}%
  \fi
  \g@addto@macro\iclr@contributionlist{\iclrmaybeauthorsup{#1}#2}%
}
\AtBeginDocument{%
  \iclrstandardauthor{%
    \parbox{\iclrAuthorBlockWidth}{\centering
      \mbox{}\\[\iclrTitleAuthorGap]
      {\linespread{\iclrAuthorLineSpread}\selectfont\normalsize \iclr@authorlist}%
      \\[\iclrAuthorAffiliationGap]
      {\normalfont\small \iclr@affiliationlist}%
      \ificlr@firstcontribution
      \else
        \\[\iclrAffiliationContributionGap]
        {\normalfont\small \iclr@contributionlist}%
      \fi
    }%
  }%
}
\makeatother

\author[1, \textdagger]{Zeyu He}
\author[1, \textdagger]{Shiqi Liu}
\author[1, \textdagger]{Ke Chen}
\author[1]{Yun Yan}
\author[1]{Jinzi Wu}
\author[1]{Dianqiao Lei}
\author[1]{\\Sirui Wang}
\author[1]{ShuRui Peng}
\author[2]{Tao Chen}
\author[2]{Zhuo Huang}
\author[2]{Yu Wu}
\author[2]{Yadong Shao}
\author[1]{\\ Zhichao Li}
\author[1]{Ke Sun}
\author[1,\Letter]{Yang Guan}
\author[1]{Keqiang Li}
\author[1,\Letter]{Shengbo Eben Li}

\affiliation[1]{School of Vehicle and Mobility \& College of AI, Tsinghua University}
\affiliation[2]{Dongfeng Motor Corporation Research \& Development Institute}
\contribution[\textdagger]{Equal contribution}
\contribution[\mbox{\textnormal{\Letter}}]{Corresponding author}

\iclrpreprintcopy 

\begin{document}
\maketitle
\begin{abstract}
Multi-modal planning is promising for autonomous driving by representing
multiple plausible behaviors in ambiguous and long-tail scenarios.
Existing methods mainly focus on improving trajectory multi-modality, enhancing
trajectory representations, or reshaping the candidate distribution.
Nevertheless, we identify a pronounced \emph{generation--evaluation asymmetry} in multi-modal planning: despite strong oracle performance, existing planners often fail to reliably select the best available candidate, leaving substantial planning potential unrealized.
To address this challenge, we propose \textbf{iDriveVLA}, a multi-modal planning framework that improves the candidate trajectory space while enabling more reliable and context-aware trajectory evaluation.
Specifically, iDriveVLA introduces a unified trajectory evaluator comprising a Safety-aware Scorer for quality and risk estimation, together with a VLM-guided Modulator for scene-adaptive criterion weighting.
We further develop an oracle-aligned progressive training strategy
consisting of candidate imitation pretraining, candidate space refinement, and semantic ranking alignment.
On the public NAVSIM v1 leaderboard, iDriveVLA achieves a new
state-of-the-art performance of \textbf{94.95 PDMS}, surpassing the human-expert reference.

\end{abstract}

\section{Introduction}
\label{sec:intro}

End-to-end autonomous driving~\citep {hu2023uniad,guan2023integrated} has rapidly advanced by jointly optimizing
perception, prediction, and planning. More recently, vision-language-action (VLA) models~\citep{jiang2025surveyvla,zhou2026opendrivevla} further introduce
rich semantic knowledge and reasoning capabilities into driving systems.
However, existing planners still struggle with challenging and long-tail
scenarios, such as partially occupied lanes, unusual obstacles, temporary road
structures, and ambiguous traffic interactions~\citep{yao2026drivesuprim}.

A key difficulty in these scenarios is that the appropriate driving behavior
is often inherently ambiguous: multiple motion strategies may be feasible,
yet they can differ substantially in safety, efficiency, comfort, and
interaction with surrounding agents~\citep{liao2025diffusiondrive}.
As illustrated in Figure\hyperref[fig:intro]{~\ref*{fig:intro}(a)}, this has motivated multi-modal
planning~\citep{guo2025ipad,li2025gtrs, guan2026enhanced}, where multiple plausible trajectories
are generated and subsequently evaluated for final execution.
Recent studies have substantially strengthened candidate generation through diffusion-based trajectory modeling~\citep{liao2025diffusiondrive}, more expressive trajectory representations~\citep{xing2026clear}, and training-distribution design for improved candidate coverage~\citep{ang2026clover}.

Despite increasingly strong candidate generation, we find that trajectory
evaluation has become a major bottleneck in multi-modal planning.
As shown in Figure\hyperref[fig:intro]{~\ref*{fig:intro}(b)}, increasing trajectory diversity improves
oracle performance by providing more potential solutions, while also
introducing more competing behaviors that the evaluator must distinguish.
This reveals a \emph{generation--evaluation asymmetry}: the ability to generate
diverse high-quality trajectories advances faster than the ability to reliably
evaluate and select among them.
This asymmetry is further quantified in Figure\hyperref[fig:intro]{~\ref*{fig:intro}(c)}, where the oracle candidate achieves
\textbf{99.21 PDMS}, whereas the learned evaluator reaches only
\textbf{93.67 PDMS} on NAVSIM v1~\citep{dauner2024navsim}.
The substantial gap indicates that considerable planning potential is already
contained in the candidate set but remains unrealized.
A key limitation is that existing evaluation mechanisms rely mainly on
predefined planning signals and lack sufficient safety awareness and
scene-adaptive semantic reasoning when comparing behaviorally distinct
trajectories.

To address this challenge, we propose \textbf{iDriveVLA}, a multi-modal
planning framework that improves both candidate quality and trajectory
selection.
iDriveVLA introduces a trajectory evaluator that combines explicit
safety awareness with VLM-guided semantic adaptation, enabling more reliable
selection among behaviorally distinct candidates.
We further develop an oracle-aligned progressive training strategy
that progressively improves candidate generation and selection reliability
through three stages: candidate imitation pretraining, candidate space refinement, and semantic ranking alignment.

Our main contributions are summarized as follows:
\begin{itemize}
    \item We identify and characterize a {generation--evaluation asymmetry} in multi-modal planning: strong candidate generators already provide high oracle potential, yet imperfect trajectory evaluation prevents this potential from being fully realized.

\item We propose {iDriveVLA}, featuring a trajectory evaluator that integrates a {Safety-aware Scorer} for low-quality candidate filtering, along with a {VLM-guided Modulator} for scene-adaptive semantic evaluation. We further develop a progressive training strategy to improve both candidate performance and ranking alignment.


\item On the public NAVSIM v1 benchmark~\citep{dauner2024navsim},
    iDriveVLA achieves {94.95 PDMS} and ranks \textbf{1st},
    surpassing the human-expert reference.
\end{itemize}

\begin{figure}[t]
    \centering
    \includegraphics[width=0.93\linewidth]{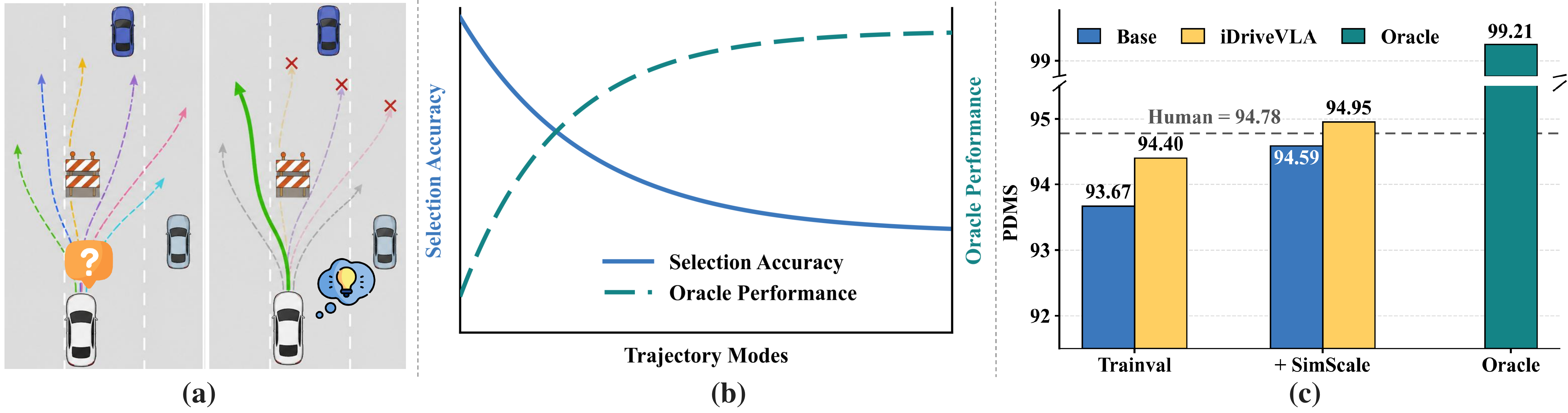}
    \caption{
    \textbf{Core idea.}
    (a) Multiple plausible trajectories may coexist in challenging scenes,
    making reliable trajectory selection critical.
    (b) Increasing trajectory diversity improves oracle performance but also
    introduces more competing behaviors for the evaluator, revealing a
    generation--evaluation asymmetry in multi-modal planning.
(c) This asymmetry leaves substantial oracle potential untapped in the
strong multi-modal planner DrivoR~\citep{kirby2026drivor}. In contrast, iDriveVLA
improves trajectory evaluation, substantially narrows the oracle gap, and
surpasses the human-expert reference on NAVSIM v1~\citep{dauner2024navsim}.
    }
    \label{fig:intro}
\end{figure}

\section{Related Work}
\label{sec:related_works}

\textbf{Vision-Language-Action Planning.}
Recent studies~\citep{jiang2025surveyvla,zhou2026opendrivevla} introduce vision-language models (VLMs) into autonomous driving to leverage their rich semantic priors and scene reasoning capabilities. A common line of work first equips VLMs with driving-specific knowledge through visual question answering or instruction tuning~\citep{chi2025impromptu}, and then extends them into VLA models that directly predict driving trajectories via action tokens~\citep{zhou2025autovla}, textual trajectory representations~\citep{hwang2025emma}, or hidden-state-based action decoders~\citep{li2026recogdrive}. Despite their stronger semantic reasoning capabilities, existing VLM-based planners often remain behind strong end-to-end driving models in trajectory-level planning performance~\citep{dauner2024navsim,jia2024bench2drive}. 

\textbf{Multi-modal Trajectory Generation.}
Recent end-to-end autonomous driving methods increasingly adopt multi-modal trajectory generation to capture diverse and uncertain driving behaviors.
Vocabulary-based approaches, such as VADv2~\citep{jiang2026vadv2}, discretize the continuous planning space into representative trajectory candidates, providing broad coverage of driving modes.
Generative methods, such as GenAD~\citep{zheng2024genad} and DiffusionDrive~\citep{liao2025diffusiondrive}, further model continuous trajectory distributions to generate diverse candidates with greater flexibility.
More recently, proposal-based methods such as iPad~\citep{guo2025ipad} and DrivoR~\citep{kirby2026drivor} employ learnable queries to iteratively generate and refine trajectory hypotheses, enlarging the optimization space beyond fixed vocabularies.
These advances have substantially improved trajectory diversity and quality; however, the resulting gains are increasingly constrained by the accuracy of downstream trajectory evaluation, limiting their full contribution to overall driving performance.

\textbf{Scorer-Based Autonomous Driving.}
Multi-modal planning has become a common paradigm in end-to-end autonomous driving, where diverse trajectory hypotheses are generated and subsequently ranked by a trajectory scorer for final execution~\citep{guo2025ipad}. Recent works improve this paradigm by learning more robust and generalizable scoring functions across different candidate distributions~\citep{li2025gtrs},
predicting interpretable planning metrics for candidate ranking~\citep{kirby2026drivor},
or learning structured scoring rules directly from driving demonstrations~\citep{xiong2025flora}.
However, their scoring criteria are primarily derived from geometric or predefined planning signals, with limited capability to adapt candidate preferences according to high-level scene semantics.

\section{Method}
\label{sec:method}

\subsection{Diagnosing the Trajectory Evaluation Bottleneck}
\label{sec:trajectory_evaluation_bottleneck}

Given multi-view visual observations $\mathcal{I}$ and the current ego state
$\bm{e}$, a multi-candidate trajectory generator produces $K$ future
trajectory hypotheses:
\begin{equation}
\mathcal{T}
=
\{\bm{\tau}_i\}_{i=1}^{K}
=
\pi_{\mathrm{gen}}(\mathcal{I},\bm{e}),
\qquad
\bm{\tau}_i\in\mathbb{R}^{T\times 3},
\label{eq:candidate_generation}
\end{equation}
where $T$ denotes the planning horizon and each trajectory consists of future
positions and headings.

\begin{wrapfigure}{r}{0.42\linewidth}
    \centering
    \vspace{-2mm}
    \includegraphics[width=\linewidth]{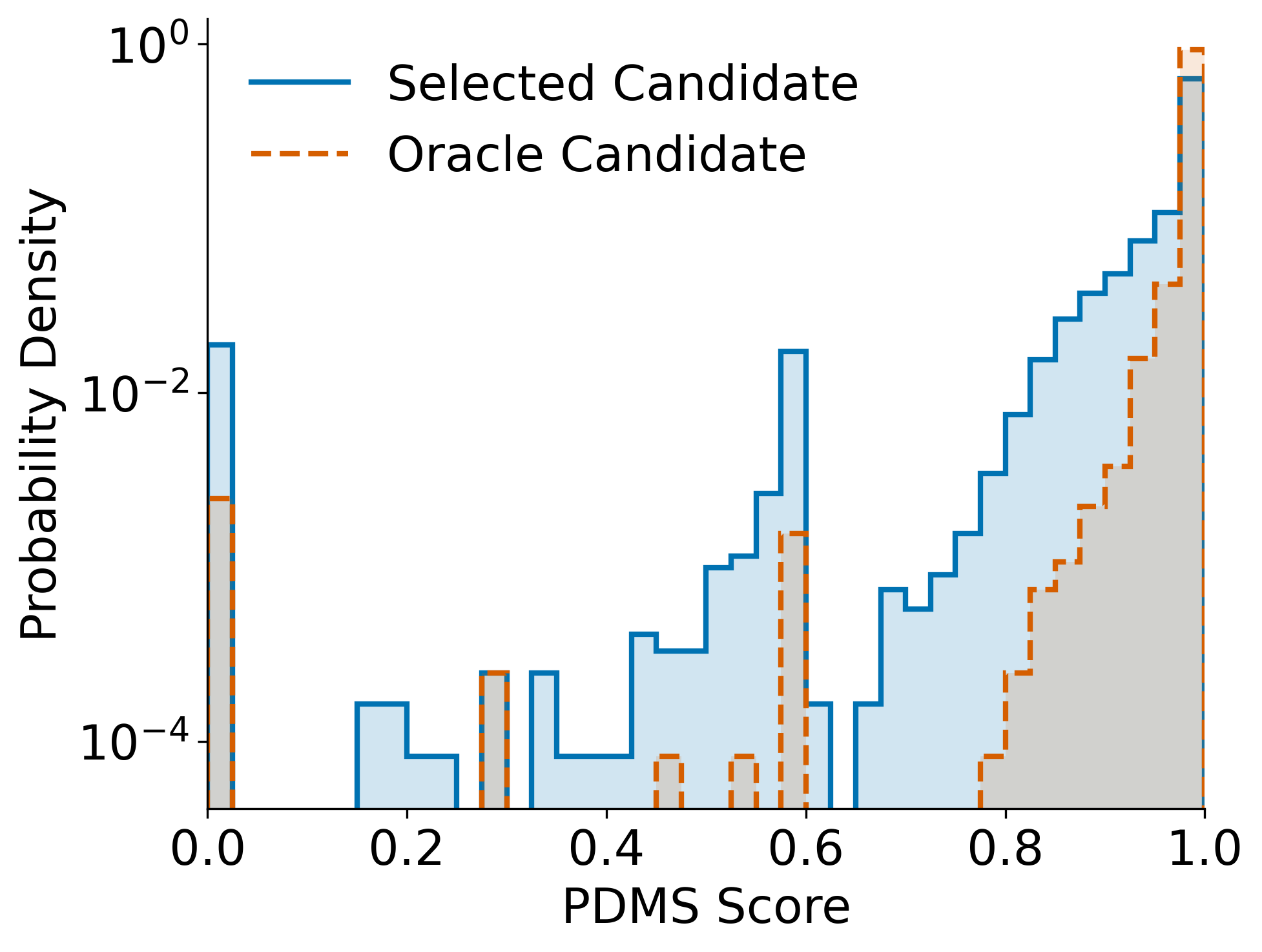}
    \caption{
    \textbf{PDMS distributions of the selected and oracle candidates from the base planner.}
    The oracle distribution is strongly concentrated near high PDMS,
    whereas the selected candidates exhibit a substantially heavier
    low-score tail.
    }
    \label{fig:candidate_selection_gap}
    \vspace{-3mm}
\end{wrapfigure}

Generating multiple candidates increases behavioral diversity and improves
the chance that a high-quality trajectory is included in
$\mathcal{T}$.
However, the final planning performance is determined not only by what the
generator can produce, but also by whether the trajectory evaluator can
identify a good candidate from the resulting set.

To diagnose this issue, we analyze the strong multi-modal driving model
DrivoR~\citep{kirby2026drivor} by comparing the downstream PDMS of its
selected trajectory with that of the {oracle candidate}, defined as the
highest-PDMS trajectory within the same candidate set.
As shown in Figure~\ref{fig:candidate_selection_gap}, oracle candidates are
strongly concentrated near high PDMS, whereas selected candidates exhibit a
broader distribution with a heavier low-score tail.
This reveals a clear \emph{generation--evaluation asymmetry}: high-quality
trajectories are often already available, but imperfect evaluation prevents
their potential from being fully realized.
As candidate diversity increases, this issue may become more pronounced
because the evaluator must distinguish among more competing behaviors.
A simple analysis of this candidate-competition effect is provided in
Appendix~\ref{app:candidate_competition_proof}.

These observations motivate us to shift attention from candidate generation
alone toward more reliable trajectory evaluation. Appendix~\ref{app:oracle_selection} further show that safe, high-quality
trajectories may already exist among geometrically similar candidates, yet
the evaluator can still select an unsafe alternative. This exposes its
insufficient sensitivity to fine-grained safety risks and interaction
semantics, directly motivating our safety-aware semantic trajectory evaluator.

\begin{figure}[t]
    \centering
\IfFileExists{figures/idrivevla_framework.pdf}{%
    \includegraphics[width=0.98\linewidth]{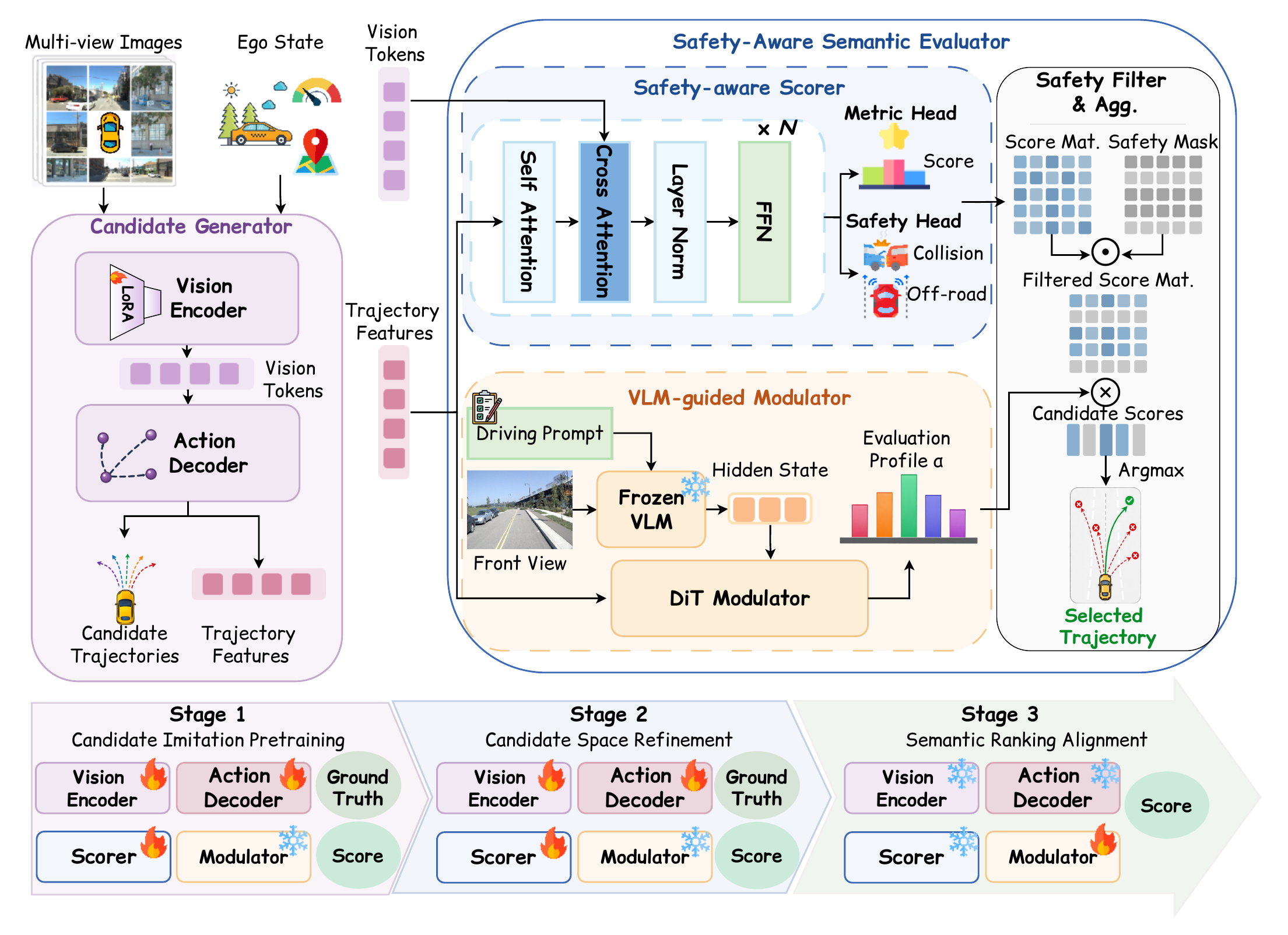}%
}{%
    \fbox{\rule{0pt}{4cm}\rule{0.95\linewidth}{0pt}}%
}
    \caption{
    \textbf{Overview of the iDriveVLA framework.}
    The candidate generator produces multiple trajectory hypotheses from
    multi-view observations and the ego state.
    Our evaluator consists of a safety-aware Scorer and a
    VLM-guided Modulator.
    The Scorer predicts criterion-wise planning scores and safety risks,
    while the Modulator generates scene-adaptive criterion weights.
    Unsafe candidates are filtered before score aggregation, and the
    highest-scoring trajectory is selected.
    The framework is progressively optimized through three training stages. Flame and snowflake icons denote trainable and frozen modules, respectively.
    }
    \label{fig:framework}
\end{figure}

\subsection{Safety-Aware Semantic Trajectory Evaluation}
\label{sec:trajectory_evaluator}

Motivated by the generation--evaluation asymmetry identified above, we improve
candidate evaluation with a Safety-aware Semantic Evaluator.
As shown in Figure~\ref{fig:framework}, it consists of a
\emph{Safety-aware Scorer} for candidate quality and safety estimation and a
\emph{VLM-guided Modulator} for scene-adaptive evaluation.
Their outputs are combined through safety filtering and score aggregation to
select the final trajectory.

\textbf{Safety-aware Scorer.}
For each candidate trajectory $\bm{\tau}_i$, the Scorer first constructs a
scene-conditioned trajectory representation:
\begin{equation}
\bm{c}_i
=
\pi_{\mathrm{score}}
\left(
\bm{\tau}_i,\bm{e},\mathcal{I}
\right).
\label{eq:candidate_representation}
\end{equation}

Following~\citep{dauner2024navsim}, we consider six structured planning
criteria,
\(
\mathcal{M}
=
\{\mathrm{NC},\mathrm{DAC},\mathrm{TTC},
\mathrm{EP},\mathrm{DDC},\mathrm{Comf.}\},
\)
corresponding to collision avoidance, drivable-area compliance,
time to collision, ego progress, driving-direction compliance, and comfort.
A metric head predicts the criterion-wise scores:
\begin{equation}
\bm{v}_i
=
\sigma\left(
\pi_{\mathrm{metric}}(\bm{c}_i)
\right),
\label{eq:metric_prediction}
\end{equation}
where $\bm{v}_i=\{v_i^m\}_{m\in\mathcal{M}}$ denotes the normalized planning
scores of candidate $\bm{\tau}_i$.

In parallel, a lightweight safety head predicts the collision and off-road
risks:
\begin{equation}
\left(
p_i^{\mathrm{col}},
p_i^{\mathrm{off}}
\right)
=
\sigma\left(
\pi_{\mathrm{safe}}(\bm{c}_i)
\right),
\label{eq:safety_prediction}
\end{equation}
which are used to suppress candidates with high predicted
risk before final ranking.

\textbf{VLM-guided Modulator.}
Fixed criterion weights cannot capture the context-dependent preference among
different driving behaviors.
We therefore introduce a semantic Modulator that adapts the evaluation
weights according to the current scene.
Given the front-view image and a predefined driving prompt, a frozen VLM
extracts a high-level semantic representation:
\begin{equation}
\bm{z}
=
\mathrm{VLM}
\left(
I_{\mathrm{front}},
\mathrm{prompt}
\right).
\label{eq:vlm_feature}
\end{equation}

Conditioned on $\bm{z}$ and the candidate features
$\{\bm{c}_i\}_{i=1}^{K}$, the DiT Modulator predicts a scene-adaptive
criterion-weight vector:
\begin{equation}
\bm{\alpha}
\sim
\pi_{\mathrm{dit}}
\left(
\cdot
\mid
\bm{z},
\{\bm{c}_i\}_{i=1}^{K}
\right),
\qquad
\bm{\alpha}
=
\{\alpha_m\}_{m\in\mathcal{M}}.
\label{eq:adaptive_weights}
\end{equation}
The Modulator changes only how different planning criteria are weighted and
does not modify the candidate trajectories themselves.

\textbf{Safety Filtering and Score Aggregation.}
For each candidate, we first construct a safety mask from the predicted
collision and off-road risks.
Unsafe candidates are removed from consideration, while the remaining
candidates are ranked by aggregating the criterion-wise scores predicted in
\eqref{eq:metric_prediction}:
\begin{equation}
M_i
=
\mathbb{I}
\left[
\max
\left(
p_i^{\mathrm{col}},
p_i^{\mathrm{off}}
\right)
\leq \delta
\right],
\qquad
S_i
=
M_i
\sum_{m\in\mathcal{M}}
\alpha_m v_i^m,
\label{eq:final_candidate_score}
\end{equation}
where $M_i\in\{0,1\}$ denotes the safety mask.
The final trajectory is then selected as
\(
\hat{\bm{\tau}}
=
\bm{\tau}_{\arg\max_i S_i}.
\label{eq:trajectory_selection}
\)

\subsection{Oracle-Aligned Progressive Training}
\label{sec:training}

To jointly improve candidate quality and selection reliability, we train
iDriveVLA in three progressive stages: candidate imitation pretraining,
candidate space refinement, and semantic ranking alignment.
These stages progressively establish diverse candidates, improve their
downstream planning quality, and align trajectory ranking with scene context.

\textbf{Stage 1: Candidate Imitation Pretraining.}
We first jointly train the trajectory generator and Scorer,
while disabling the VLM-guided Modulator. Given the expert trajectory
$\bm{\tau}^{\star}$, the generator is optimized with a best-of-$K$
imitation objective:
\begin{equation}
\mathcal{L}_{\mathrm{traj}}
=
\min_{i=1,\ldots,K}
\left\|
\bm{\tau}_{i}
-
\bm{\tau}^{\star}
\right\|_1 .
\label{eq:stage1_traj}
\end{equation}
This objective encourages the candidate set to contain trajectories close to
expert behaviors while preserving diverse motion hypotheses.

Meanwhile, the Scorer is supervised using candidate-level planning metrics
and safety labels. Let $g_i^m$ denote the ground-truth value of criterion
$m\in\mathcal{M}$, and define
$b_i^{\mathrm{col}}=\mathbb{I}[g_i^{\mathrm{NC}}=0]$ and
$b_i^{\mathrm{off}}=\mathbb{I}[g_i^{\mathrm{DAC}}=0]$.
The evaluator loss is:
\begin{equation}
\mathcal{L}_{\mathrm{eval}}
=
\frac{1}{K}
\sum_{i=1}^{K}
\left[
\sum_{m\in\mathcal{M}}
\operatorname{BCE}(v_i^m,g_i^m)
+
\operatorname{BCE}
(p_i^{\mathrm{col}},b_i^{\mathrm{col}})
+
\operatorname{BCE}
(p_i^{\mathrm{off}},b_i^{\mathrm{off}})
\right].
\label{eq:score_loss}
\end{equation}

Combining the trajectory imitation objective with the evaluator supervision
in \eqref{eq:score_loss}, the overall Stage-1 objective is:
\begin{equation}
\mathcal{L}_{\mathrm{S1}}
=
\mathcal{L}_{\mathrm{traj}}
+
\lambda_{\mathrm{eval}}
\mathcal{L}_{\mathrm{eval}},
\label{eq:stage1_loss}
\end{equation}
where $\lambda_{\mathrm{eval}}$ balances trajectory imitation and structured
evaluator supervision.

\textbf{Stage 2: Candidate Space Refinement.}
Although Stage 1 establishes a diverse candidate space, the generated
distribution may still contain trajectories that are suboptimal under the
downstream planning objective. We therefore further refine the candidate
generator according to planning rewards. Inspired by reward-guided trajectory
generation in diffusion-based planning~\citep{zheng2025diffusionplanner}, we
introduce a reward-weighted supervised refinement (RFT) objective to emphasize
training samples associated with higher downstream planning quality.

Specifically, let $R^\star\in[0,1]$ denote the PDMS reward of the expert trajectory
$\bm{\tau}^{\star}$. Building on the best-of-$K$ imitation objective in
\eqref{eq:stage1_traj}, we weight each training sample according to the
downstream quality of its expert trajectory:
\begin{equation}
w_{\mathrm{RFT}}
=
\exp(\beta R^\star),
\qquad
\mathcal{L}_{\mathrm{RFT}}
=
w_{\mathrm{RFT}}
\mathcal{L}_{\mathrm{traj}},
\label{eq:rft_objective}
\end{equation}
where $\beta$ controls the strength of reward-based refinement.
This exponential weighting can also be interpreted as reward-based
exponential tilting of the supervised training distribution under a
KL-regularized objective; we provide the corresponding derivation in
Appendix~\ref{app:rft_derivation}.

The complete Stage-2 objective is:
\begin{equation}
\mathcal{L}_{\mathrm{S2}}
=
\mathcal{L}_{\mathrm{RFT}}
+
\lambda_{\mathrm{eval}}
\mathcal{L}_{\mathrm{eval}} .
\label{eq:stage2_loss}
\end{equation}
This stage places greater optimization emphasis on expert trajectories with
higher downstream planning quality while retaining structured evaluator
supervision.




\textbf{Stage 3: Semantic Ranking Alignment.}
Even with improved candidate quality, reliable trajectory selection requires
understanding scene-dependent driving preferences. Therefore, we optimize the
VLM-guided Modulator as a semantic ranking policy, while freezing the
trajectory generator, Scorer, and VLM backbone.

Conditioned on the semantic representation $\bm{z}$ and candidate features
$\{\bm{c}_i\}_{i=1}^{K}$, it samples multiple semantic ranking profiles:
\begin{equation}
\bm{\alpha}^{(g)}
\sim
\pi_{\mathrm{dit}}
\left(
\cdot
\mid
\bm{z},
\{\bm{c}_i\}_{i=1}^{K}
\right),
\qquad
g=1,\ldots,G .
\label{eq:vlm_profile_sampling}
\end{equation}

Each sampled profile $\bm{\alpha}^{(g)}$ produces candidate scores
$S_i^{(g)}$ according to \eqref{eq:final_candidate_score}.
We convert these scores into a soft selection distribution over the top-$k$
candidates and compute the corresponding planning reward:
\begin{equation}
R_g
=
\sum_{i\in\mathcal{I}_{k}^{(g)}}
\omega_i^{(g)}
\mathcal{R}_{\mathrm{PDMS}}(\bm{\tau}_i),
\qquad
\omega_i^{(g)}
=
\frac{
\exp(S_i^{(g)})
}{
\sum_{j\in\mathcal{I}_{k}^{(g)}}
\exp(S_j^{(g)})
}.
\label{eq:soft_topk_reward}
\end{equation}
where $\mathcal{I}_{k}^{(g)}$ denotes the indices of the top-$k$ candidates ranked by $S_i^{(g)}$, and $\mathcal{R}_{\mathrm{PDMS}}(\bm{\tau}_i)$ denotes the PDMS score obtained by evaluating candidate trajectory $\bm{\tau}_i$~\citep{guan2021direct}.
The rewards defined in \eqref{eq:soft_topk_reward} are standardized within
each sampled group and used to optimize the Modulator with
STAPO~\citep{liu2026stapo}:
\begin{equation}
\begin{aligned}
\mathcal{J}_{\mathrm{STAPO}}(\theta)
&=
\mathbb{E}_{g}
\left[
\min
\left(
\rho_g(\theta)A_g,
\operatorname{clip}
\left(
\rho_g(\theta),
1-\epsilon_{\mathrm{low}},
1+\epsilon_{\mathrm{high}}
\right)
A_g
\right)
\right], \\
A_g
&=
\frac{R_g-\mu_R}{\sigma_R},
\qquad
\rho_g(\theta)
=
\frac{
\pi_{\mathrm{dit}}
\left(
\bm{\alpha}^{(g)}
\mid
\bm{z},\{\bm{c}_i\}_{i=1}^{K}
\right)
}{
\pi_{\mathrm{dit}_{\mathrm{old}}}
\left(
\bm{\alpha}^{(g)}
\mid
\bm{z},\{\bm{c}_i\}_{i=1}^{K}
\right)
},
\end{aligned}
\label{eq:stapo_objective}
\end{equation}
where $\theta$ denotes the parameters of the Modulator, $\mu_R$ and $\sigma_R$ denote the within-group mean and standard deviation of the sampled rewards, respectively, and $\rho_g(\theta)$ denotes the importance ratio between the current and behavior policies, computed following~\citep{li2026recogdrive}. STAPO further masks abnormal gradients induced by outlier policy
updates, thereby stabilizing the optimization process. By optimizing semantic ranking policies with downstream planning rewards,
the Modulator learns to adapt candidate preferences according to scene
semantics~\citep{li2023reinforcement}.


\section{Experiments}
\subsection{Settings}
\textbf{Experimental Setup.}
For trajectory generation, we adopt the multi-modal planner of DrivoR~\citep{kirby2026drivor} and follow its standard configuration. For semantic modulation, we use a frozen InternVL3-2B initialized with the pretrained weights from ReCogDrive~\citep{li2026recogdrive}. All models are trained on 8 NVIDIA A800 GPUs following a three-stage training procedure, with RADAR~\citep{ren2026momentum} used as the optimizer. In Stage 1, we train the model for 25 epochs with a
learning rate of $2\times10^{-4}$. In Stage 2, the learning rate is
reduced to $1\times10^{-5}$, and the model is trained for an additional
30 epochs. In the third stage, we train the VLM-conditioned diffusion module
for 25 epochs with a learning rate of $1\times10^{-5}$. Further training and optimization details are provided in Appendix~\ref{app:Implementation details}.

\textbf{Dataset.}
The model is trained primarily on NAVSIM~\citep{dauner2024navsim}, a planning-oriented autonomous driving dataset from OpenScene. NAVSIM contains approximately 103K training samples and 12K test samples. To improve data diversity, we augment its real-world training set with 134K simulated samples generated by SimScale~\citep{tian2026simscale}. These samples are obtained by perturbing ego trajectories and rendering them in reactive environments.

\textbf{Metrics.}
We evaluate trajectory quality through closed-loop simulation in NAVSIM,
reporting the Predictive Driver Model Score (PDMS) for NAVSIM v1 and the
Extended Predictive Driver Model Score (EPDMS) for NAVSIM v2. Detailed
definitions and formulations are provided in
Appendix~\ref{app:evaluation_metrics}. 
\begin{table*}[!t]
\centering
\caption{
\textbf{Closed-loop evaluation on NAVSIM v1.}
Higher values indicate better performance.
}
\label{tab:navsim_v1}

\definecolor{rowblueDark}{RGB}{225,238,248}

\footnotesize
\setlength{\tabcolsep}{4.5pt}
\renewcommand{\arraystretch}{1.05}

\begin{tabularx}{\textwidth}{
>{\raggedright\arraybackslash}X
*{5}{>{\centering\arraybackslash}p{1.15cm}}
|
>{\centering\arraybackslash}p{1.15cm}
}
\toprule

\textbf{Method}
& \textbf{NC$\uparrow$}
& \textbf{DAC$\uparrow$}
& \textbf{TTC$\uparrow$}
& \textbf{Comf.$\uparrow$}
& \textbf{EP$\uparrow$}
& \textbf{PDMS$\uparrow$} \\

\midrule

\rowcolor{gray!15}
PDM-Closed~\citep{dauner2023parting}
& 94.6 & 99.8 & 89.9 & 86.9 & 99.9 & 89.1 \\

\rowcolor{gray!15}
Human driver~\citep{dauner2024navsim}
& 100.0 & 100.0 & 100.0 & 99.9 & 87.5 & 94.8 \\

\midrule
\multicolumn{7}{l}{\textit{E2E Methods}} \\
\midrule

UniAD~\citep{hu2023uniad}
& 97.8 & 91.9 & 92.9 & 100.0 & 78.8 & 83.4 \\

LTF~\citep{chitta2023transfuser}
& 97.4 & 92.8 & 92.4 & 100.0 & 79.0 & 83.8 \\

PARA-Drive~\citep{weng2024paradrive}
& 97.9 & 92.4 & 93.0 & 99.8 & 79.3 & 84.0 \\

DriveX-S~\citep{shi2025drivex}
& 97.5 & 94.0 & 93.0 & 100.0 & 79.7 & 84.5 \\

World4Drive~\citep{zheng2025world4drive}
& 97.4 & 94.3 & 92.8 & 100.0 & 79.9 & 85.1 \\

DRAMA~\citep{yuan2024drama}
& 98.0 & 93.1 & 94.8 & 100.0 & 80.1 & 85.5 \\

DiffusionDrive~\citep{liao2025diffusiondrive}
& 98.2 & 96.2 & 94.7 & 100.0 & 82.2 & 88.1 \\

iPad~\citep{guo2025ipad}
& 98.6 & 98.3 & 94.9 & 100.0 & 88.0 & 91.7 \\

DrivoR~\citep{kirby2026drivor}
& 99.0 & 98.9 & 96.7 & 100.0 & 90.0 & 93.7 \\

RAP~\citep{feng2026rap}
& \underline{99.1} & 98.9 & 96.7 & 100.0 & 90.3 & 93.8 \\

CLOVER~\citep{ang2026clover}
& \underline{99.1} & 99.0 & 96.9 & 100.0 & 91.7 & 94.5 \\

TOAD~\citep{xu2026toad}
& 99.0 & \textbf{99.3} & 96.8 & 100.0 & 91.8 & 94.7 \\

\midrule
\multicolumn{7}{l}{\textit{VLA Methods}} \\
\midrule

UniVLA~\citep{wang2025univla}
& 96.9 & 91.1 & 91.7 & 96.7 & 76.8 & 81.7 \\

FSDrive~\citep{zeng2025fsdrive}
& 98.2 & 93.8 & 93.3 & 99.9 & 80.1 & 85.1 \\

AutoVLA~\citep{zhou2025autovla}
& 98.4 & 95.6 & \textbf{98.0} & \underline{99.9} & 81.9 & 89.1 \\

DriveVLA-W0~\citep{li2026drivevlaw0}
& 98.7 & \underline{99.1} & 95.3 & 99.3 & 83.3 & 90.2 \\

AdaThinkDrive~\citep{luo2026adathinkdrive}
& 98.4 & 97.8 & 95.2 & 100.0 & 84.4 & 90.3 \\

SpanVLA~\citep{zhou2026spanvla}
& \underline{99.1} & 97.1 & 95.2 & 100.0 & 86.3 & 90.3 \\

ReCogDrive~\citep{li2026recogdrive}
& 97.9 & 97.3 & 94.9 & 100.0 & 87.3 & 90.8 \\

ELF-VLA~\citep{luo2026elfvla}
& 98.9 & 98.1 & 96.0 & 100.0 & 85.3 & 91.0 \\

SGDrive~\citep{li2026sgdrive}
& 98.6 & 97.8 & 96.2 & 100.0 & 85.8 & 91.1 \\

FLARE~\citep{xie2026flare}
& 98.5 & 98.4 & 96.0 & 100.0 & 86.0 & 91.4 \\

CLEAR~\citep{xing2026clear}
& \underline{99.1}
& 98.8
& \underline{97.2}
& 99.6
& 89.7
& 93.7 \\

ChainFlow-VLA~\citep{wang2026chainflowvla}
& \textbf{99.2}
& 99.0
& \underline{97.2}
& 99.9
& \underline{91.9}
& \underline{94.8} \\

\rowcolor{rowblueDark}
\textbf{iDriveVLA}
& 98.9
& \textbf{99.3}
& 96.4
& \textbf{100.0}
& \textbf{92.9}
& \textbf{94.9} \\

\bottomrule
\end{tabularx}
\end{table*}

\begin{table*}[!t]
\centering
\caption{
\textbf{Closed-loop evaluation on NAVSIM v2.}
All results are computed using the original NAVSIM v2 evaluation code
without the human-behavior filtering fix.
}
\label{tab:navsim_v2}

\definecolor{rowblueDark}{RGB}{225,238,248}

\setlength{\tabcolsep}{4.0pt}
\renewcommand{\arraystretch}{1.08}

\resizebox{\textwidth}{!}{
\begin{tabular}{lccccccccc|c}
\toprule

\textbf{Method}
& \textbf{NC$\uparrow$}
& \textbf{DAC$\uparrow$}
& \textbf{DDC$\uparrow$}
& \textbf{TLC$\uparrow$}
& \textbf{EP$\uparrow$}
& \textbf{TTC$\uparrow$}
& \textbf{LK$\uparrow$}
& \textbf{HC$\uparrow$}
& \textbf{EC$\uparrow$}
& \textbf{EPDMS$\uparrow$} \\

\midrule
\multicolumn{11}{l}{\textit{E2E Methods}} \\
\midrule

Ego Status~\citep{cao2025pseudosimulation}
& 93.1 & 77.9 & 92.7 & 99.6 & 86.0
& 91.5 & 89.4 & 98.3 & 85.4 & 64.0 \\

TransFuser~\citep{chitta2023transfuser}
& 96.9 & 89.9 & 97.8 & 99.7 & 87.1
& 95.4 & 92.7 & 98.3 & \underline{87.2} & 76.7 \\

DiffusionDrive~\citep{liao2025diffusiondrive}
& 98.2 & 95.9 & \underline{99.4} & \underline{99.8} & 87.5
& 97.3 & \textbf{96.8} & 98.3 & \textbf{87.7} & 84.5 \\

Hydra-MDP++~\citep{li2025hydramdpplusplus}
& 98.5 & 98.5 & 99.5 & 99.7 & 87.4
& 97.9 & 95.8 & \underline{98.2} & 75.7 & 85.6 \\

CLOVER~\citep{ang2026clover}
& \textbf{99.4} & 99.3 & \textbf{99.5} & \textbf{99.8} & 86.9
& \textbf{98.9} & 95.2 & 98.3 & 75.5 & \underline{87.2} \\

\midrule
\multicolumn{11}{l}{\textit{VLA Methods}} \\
\midrule

DriveVLA-W0~\citep{li2026drivevlaw0}
& 98.5 & \underline{99.1} & 98.0 & 99.7 & 86.4
& 98.1 & 93.2 & 97.9 & 58.9 & 86.1 \\

ReCogDrive~\citep{li2026recogdrive}
& 98.3 & 95.2 & 99.5 & 99.8 & 87.1
& 97.5 & \underline{96.6} & 98.3 & 86.5 & 83.6 \\

LaST-VLA~\citep{luo2026lastvla}
& 98.7 & 97.9 & 99.2 & 99.7 & \underline{90.3}
& 98.2 & \underline{96.6} & 98.3 & 86.3 & 87.1 \\

\rowcolor{rowblueDark}
\textbf{iDriveVLA}
& \underline{98.9}
& \textbf{99.3}
& 99.0
& 96.8
& \textbf{91.6}
& \underline{98.4}
& 85.3
& \textbf{98.3}
& 77.9
& \textbf{87.4} \\

\bottomrule
\end{tabular}
}
\end{table*}

\subsection{Main Results}
As shown in Table~\ref{tab:navsim_v1}, iDriveVLA achieves the \textbf{highest}  PDMS on NAVSIM v1, reaching \textbf{94.9} and ranking \textbf{1st} among all compared methods. It outperforms both TOAD, the strongest existing end-to-end driving method, and ChainFlow-VLA, the best-performing VLA-based approach. These results demonstrate that improving trajectory evaluation can effectively unlock the potential of multi-modal trajectory planning, rather than relying solely on stronger trajectory generation. Notably, iDriveVLA also surpasses the human-driver reference score of 94.8 PDMS. This further suggests that incorporating safety-aware evaluation and VLM-based semantic modulation can improve the planner's ability to identify high-quality trajectories and push overall driving performance beyond the human-reference level on this benchmark.

To further evaluate iDriveVLA under a more comprehensive driving protocol, we report results on NAVSIM v2, which extends PDMS with additional criteria for driving-direction and traffic-light compliance, lane keeping, and temporal comfort. As shown in Table~\ref{tab:navsim_v2}, iDriveVLA achieves state-of-the-art performance with an EPDMS of 87.4. Compared with LaST-VLA and CLOVER, iDriveVLA maintains competitive safety and compliance performance while achieving higher ego progress, indicating robust trajectory selection under the more demanding evaluation protocol.

\subsection{Ablation Studies}
We further analyze the contributions of the three key components of iDriveVLA in Table~\ref{tab:ablation}. Starting from the DrivoR baseline with a PDMS of 93.7, introducing the Safety-aware Scorer improves performance to 93.9 by suppressing risky candidates without sacrificing driving efficiency. Candidate space refinement further raises PDMS to 94.2, demonstrating that refinement fine-tuning improves the quality of the candidate set before trajectory selection. Finally, adding the VLM-guided Modulator achieves the best overall performance by incorporating high-level scene semantics into trajectory evaluation, enabling more context-adaptive candidate evaluation.

\begin{table}[!t]
\centering
\caption{
\textbf{Ablation study of the proposed components on NAVSIM v1.}
Safety Gate denotes the Safety-aware Scorer, Refinement denotes candidate space refinement fine-tuning, and Modulator denotes the VLM-guided Modulator. All variants are trained exclusively on NAVSIM data.
}
\label{tab:ablation}

\definecolor{ablue}{RGB}{240,247,252}

\resizebox{\linewidth}{!}{
\begin{tabular}{cccccccccc}
\toprule

\multicolumn{3}{c}{\textbf{Components}}
&
\multicolumn{6}{c}{\textbf{Planning Metrics}}
&
\multirow{2}{*}{\textbf{$\boldsymbol{\Delta}$ PDMS}}
\\

\cmidrule(lr){1-3}
\cmidrule(lr){4-9}

\textbf{Safety Gate}
& \textbf{Refinement}
& \textbf{Modulator}
& \textbf{NC$\uparrow$}
& \textbf{DAC$\uparrow$}
& \textbf{TTC$\uparrow$}
& \textbf{Comf.$\uparrow$}
& \textbf{EP$\uparrow$}
& \textbf{PDMS$\uparrow$}
& \\

\midrule

\rowcolor{gray!10}
\multicolumn{3}{c}{\textbf{Base Planner}}
& 99.0 & 98.9 & 96.7 & 100.0 & 90.0 & 93.7 & -- \\

\checkmark
&
&
& 99.1 & 98.9 & 96.7 & 100.0 & 91.1 & 93.9 & +0.2 \\

\checkmark
& \checkmark
&
& 99.0 & 99.0 & 96.6 & 100.0 & 91.3 & 94.2 & +0.5 \\

\rowcolor{ablue}
\checkmark
& \checkmark
& \checkmark
& 99.1 & 99.1 & 96.6 & 100.0 & 91.6 & 94.4 & +0.7 \\

\bottomrule
\end{tabular}
}
\end{table}
\subsection{Qualitative Results}
\begin{figure*}[!t]
  \centering
  \begin{minipage}[t]{0.495\textwidth}
    \centering
    \includegraphics[width=\linewidth]{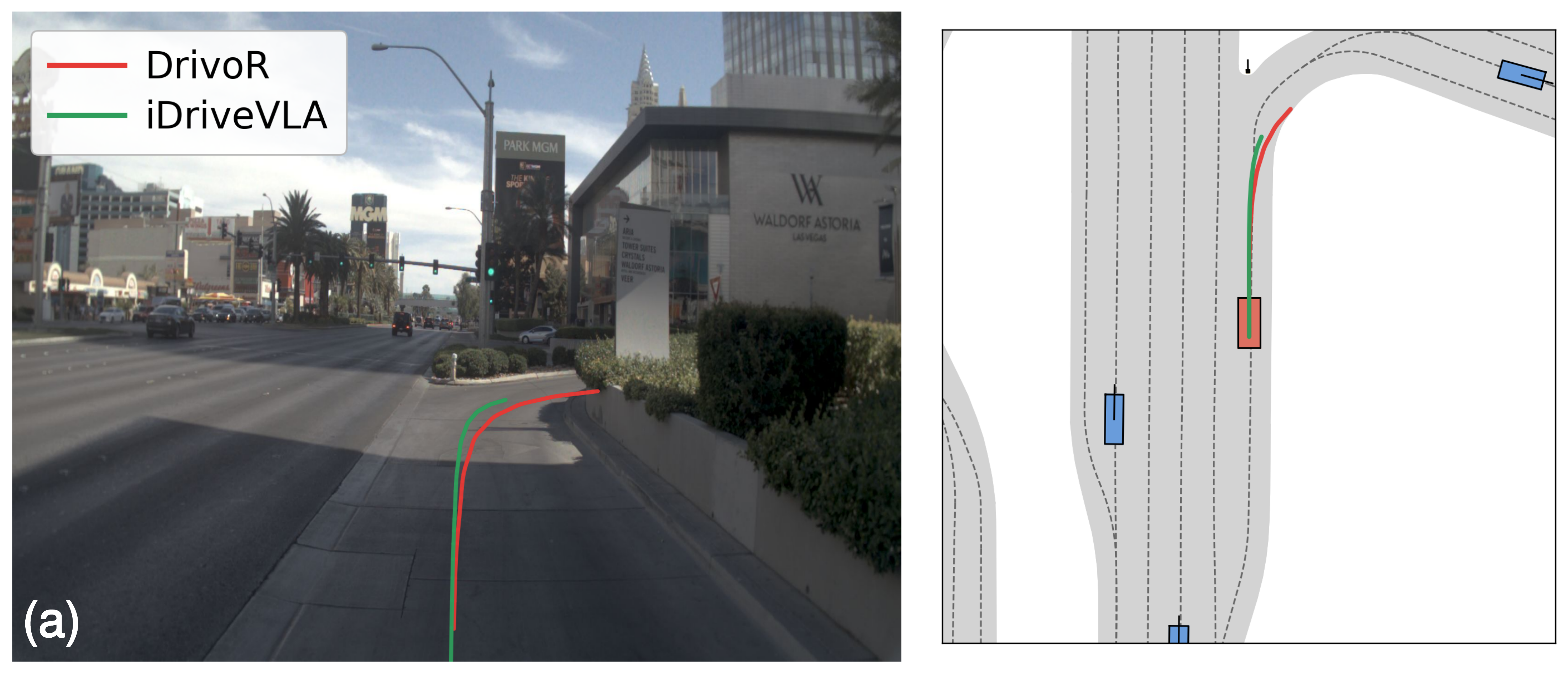}
  \end{minipage}\hfill
  \begin{minipage}[t]{0.495\textwidth}
    \centering
    \includegraphics[width=\linewidth]{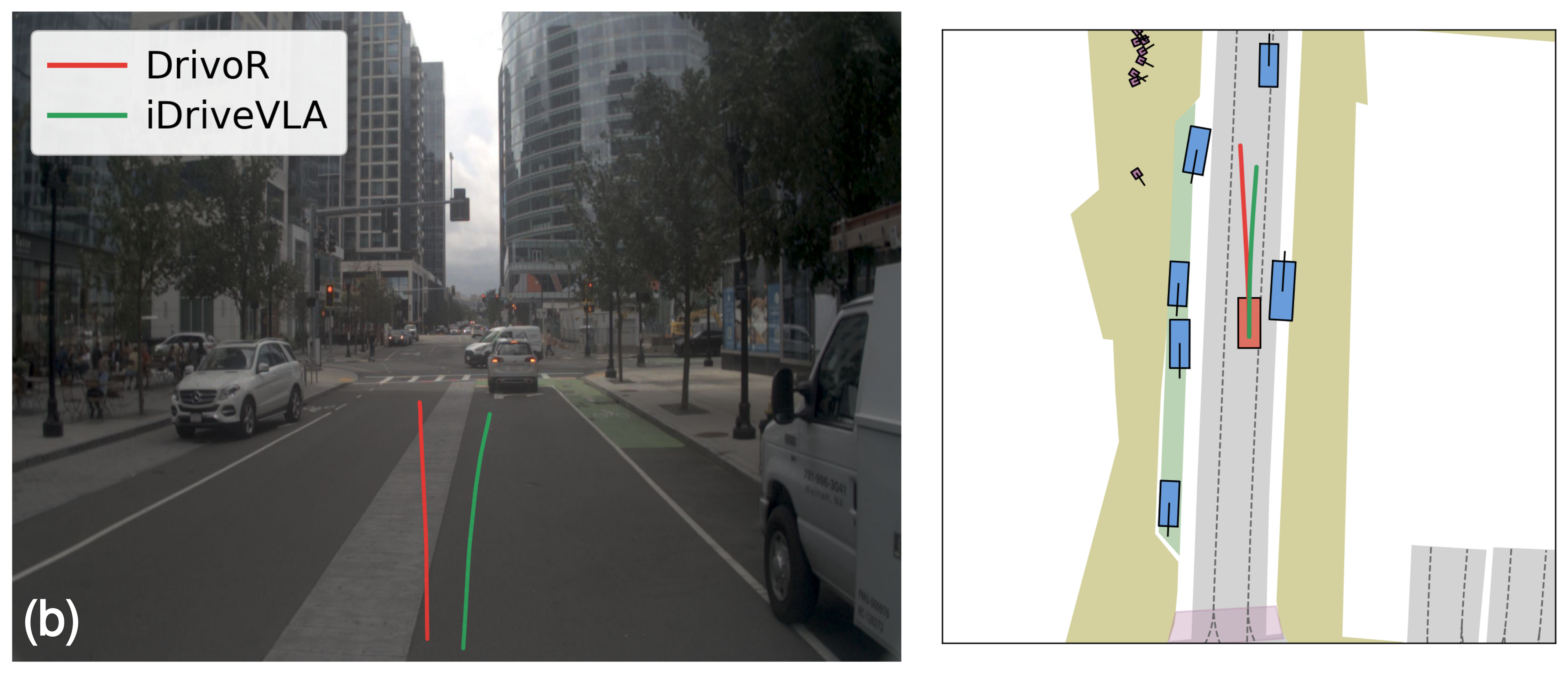}
  \end{minipage}
  \begin{minipage}[t]{0.495\textwidth}
    \centering
    \includegraphics[width=\linewidth]{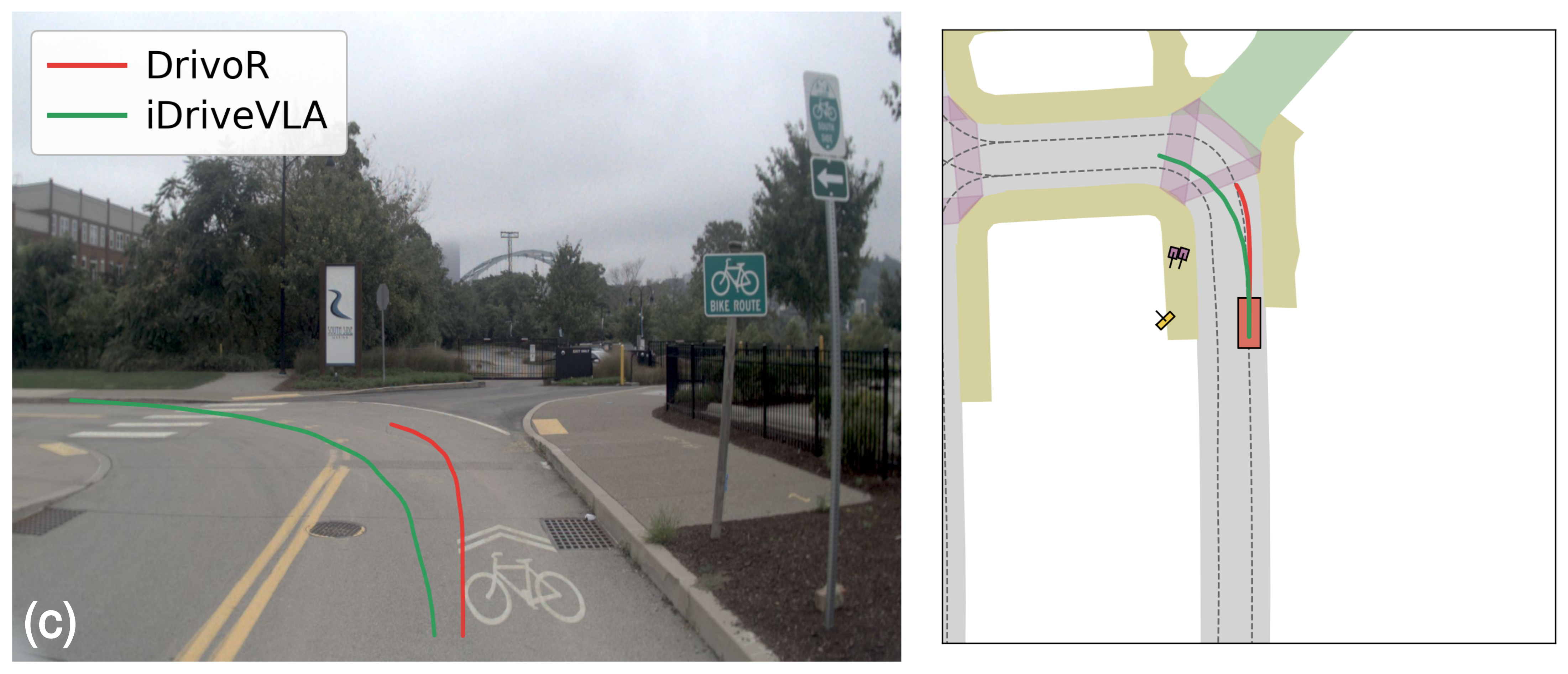}
  \end{minipage}\hfill
  \begin{minipage}[t]{0.495\textwidth}
    \centering
    \includegraphics[width=\linewidth]{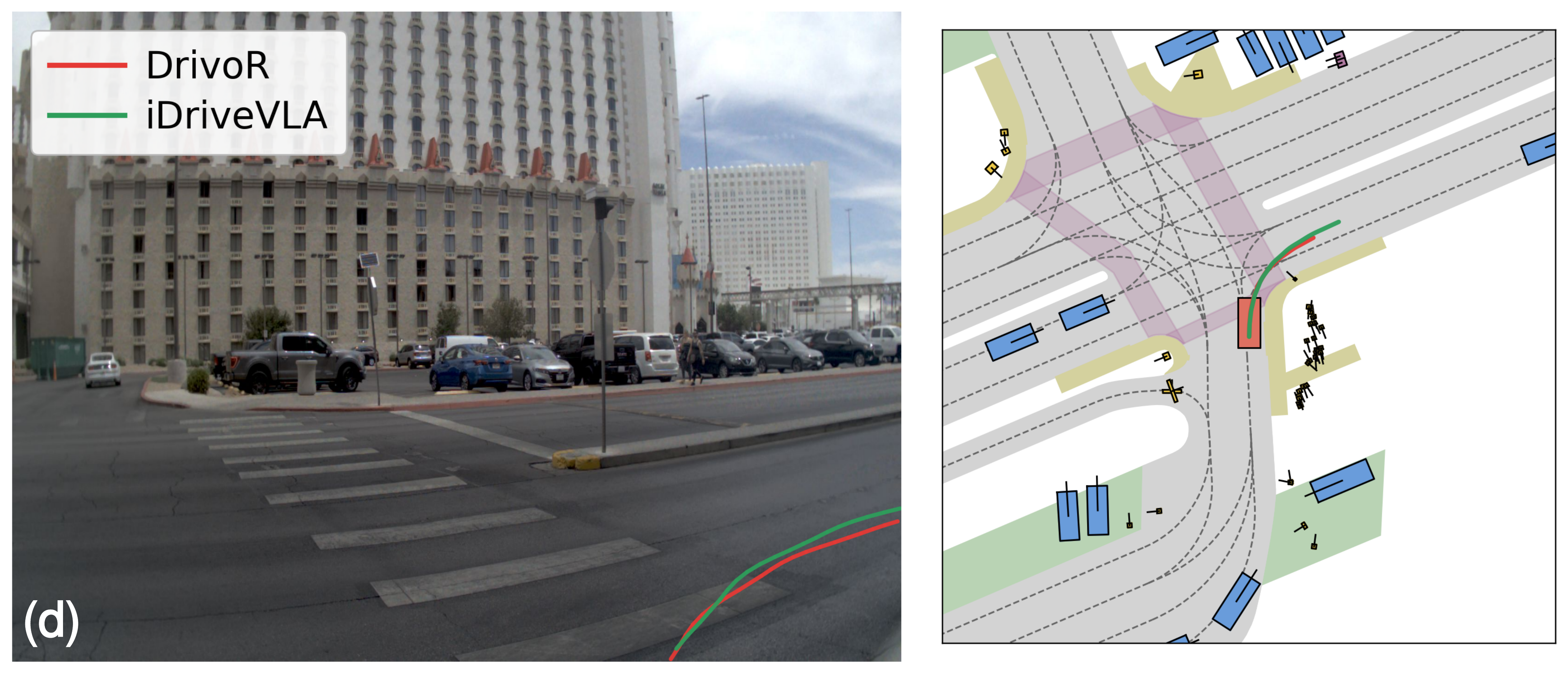}
  \end{minipage}
  \caption{Comparison between DrivoR and iDriveVLA. Red and green curves
  denote the DrivoR and iDriveVLA trajectories, respectively. The paired camera
  and bird's-eye-view panels show how the predicted paths differ under varied
  road geometry and surrounding traffic.}
  \label{fig:qualitative_pairwise}
\end{figure*}
\begin{figure*}[!t]
  \centering
  \begin{minipage}[t]{0.495\textwidth}
    \centering
    \includegraphics[width=\linewidth]{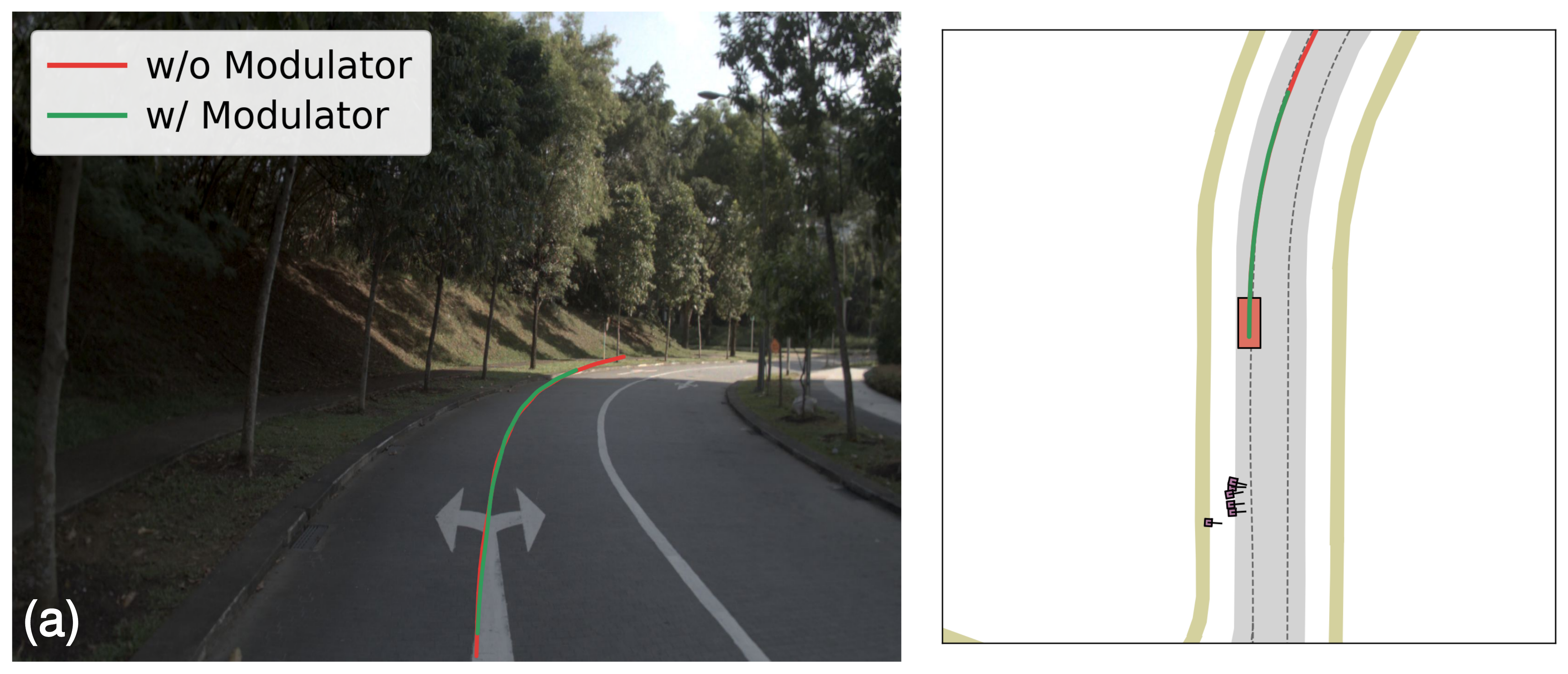}
  \end{minipage}\hfill
  \begin{minipage}[t]{0.495\textwidth}
    \centering
    \includegraphics[width=\linewidth]{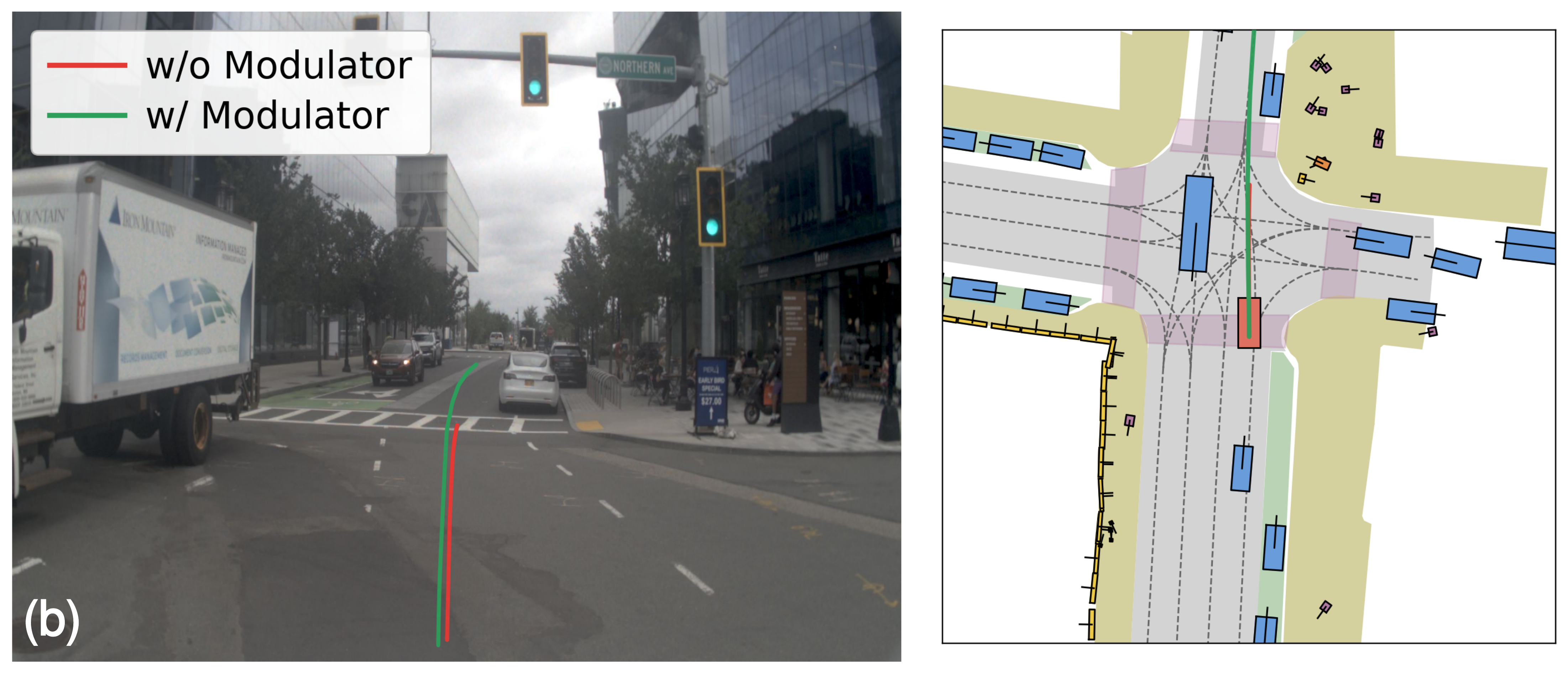}
  \end{minipage}

  \begin{minipage}[t]{0.495\textwidth}
    \centering
    \includegraphics[width=\linewidth]{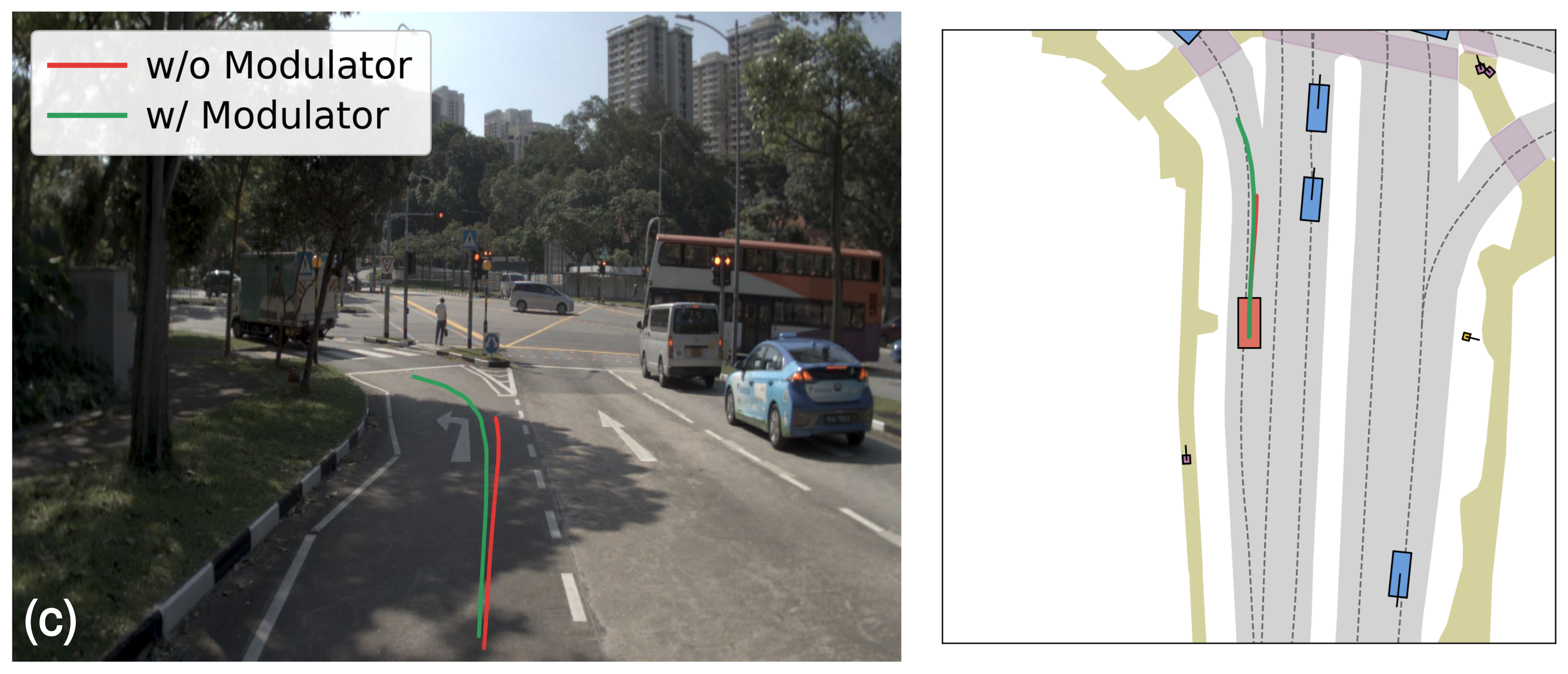}
  \end{minipage}\hfill
  \begin{minipage}[t]{0.495\textwidth}
    \centering
    \includegraphics[width=\linewidth]{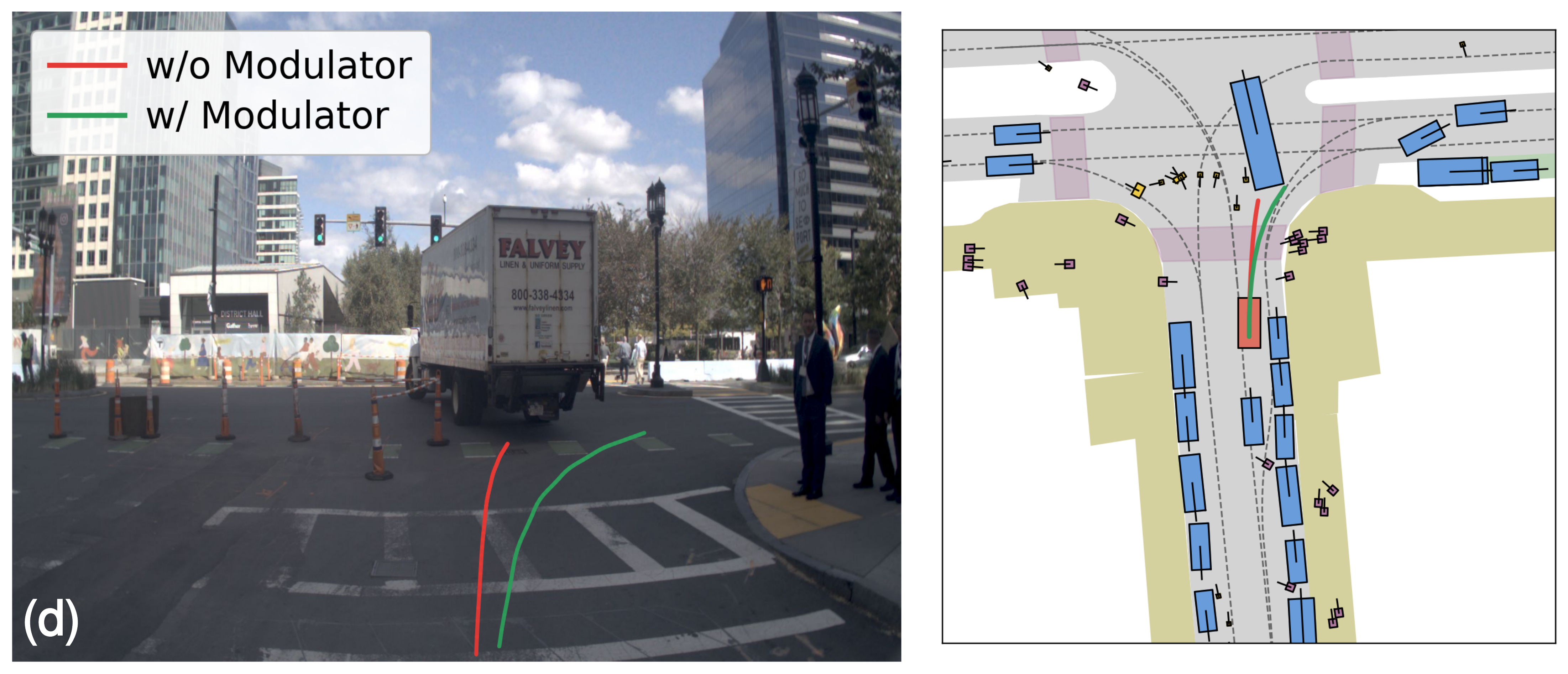}
  \end{minipage}
  \caption{Qualitative comparison of trajectory selection without and with the
  VLM-guided Modulator. Red and green curves denote the selected trajectories
  without and with the Modulator, respectively. The Modulator changes the
  ranking of existing candidates according to the surrounding road structure
  and traffic context.}
  \label{fig:qualitative_modulator}
\end{figure*}
We further visualize the trajectories selected by iDriveVLA and the base planner DrivoR in representative driving scenarios, as shown in Figure~\ref{fig:qualitative_pairwise}. The red and green curves denote the trajectories selected by DrivoR and iDriveVLA, respectively. In Figures\hyperref[fig:qualitative_pairwise]{~\ref*{fig:qualitative_pairwise}(a)}
and \hyperref[fig:qualitative_pairwise]{\ref*{fig:qualitative_pairwise}(b)}, DrivoR selects trajectories that lead to collisions or departures from the drivable area, whereas iDriveVLA avoids these failures while maintaining forward progress. As Figures\hyperref[fig:qualitative_pairwise]{~\ref*{fig:qualitative_pairwise}(c)}
and \hyperref[fig:qualitative_pairwise]{\ref*{fig:qualitative_pairwise}(d)} show, DrivoR exhibits overly conservative behavior during turning and lane-changing maneuvers, resulting in limited progress. In contrast, iDriveVLA selects smoother and more appropriate trajectories that remain within the drivable area while advancing efficiently. These examples qualitatively demonstrate that context-aware trajectory evaluation enables iDriveVLA to select safer and more effective behaviors across diverse driving scenarios.

We also investigate the effect of the VLM-guided Modulator on trajectory selection in Figure~\ref{fig:qualitative_modulator}. With the candidate set held fixed, the Modulator alters only the selection preference. In Figure\hyperref[fig:qualitative_modulator]{~\ref*{fig:qualitative_modulator}(a)}, the trajectory selected with modulation, shown in green, follows the lane structure more closely while maintaining greater clearance from nearby obstacles. The remaining three cases further show that scene-level semantic guidance helps avoid overly conservative choices while preserving the existing safety calibration. Overall, the Modulator improves driving progress without introducing additional collisions or drivable-area violations in these examples. Additional comparisons with oracle trajectories are provided in
Appendix~\ref{app:oracle_selection}, showing that iDriveVLA can identify
optimal or near-optimal candidates.

\section{Conclusion}

Our work revisits multi-modal planning from the perspective of trajectory evaluation rather than candidate generation alone. Our results show that a considerable portion of planning performance is already contained in the candidate set, but remains inaccessible when the planner cannot reliably distinguish high-quality trajectories from suboptimal ones. Building on this observation, we develop iDriveVLA to jointly improve candidate quality and trajectory selection through safety-aware scoring, semantic modulation, and oracle-aligned training. Extensive experiments on NAVSIM v1 and v2 demonstrate consistent gains across overall planning quality, safety, compliance, and driving progress, with iDriveVLA achieving state-of-the-art performance on both benchmarks. 

Despite these improvements, a substantial gap between selected and oracle performance still remains, indicating that the potential of existing candidate sets is far from fully exploited.  We hope iDriveVLA can serve as a step toward establishing trajectory evaluation as a more explicit research direction in multi-modal autonomous driving and inspire future work on more reliable, adaptive, and semantically informed trajectory evaluation.

\subsection*{AI use statement}

Generative AI tools were employed in a limited supporting role during the preparation of this manuscript. Their use was restricted to improving linguistic clarity and presentation, assisting with routine code debugging and small code revisions, and helping prepare or refine illustrative schematic figures. They were not used to generate scientific content or to make substantive research decisions, including the formulation of research questions or hypotheses, conceptual or theoretical development, mathematical reasoning, methodology and experimental design, implementation of the proposed method, data construction or processing, qualitative assessment, or interpretation of experimental results.

All AI-assisted outputs were subsequently inspected and revised by the authors. Any modified code was validated within the experimental workflow, and all generated or edited figures were checked for consistency with the corresponding technical descriptions and results. The authors independently developed and verified the scientific ideas, methods, experiments, analyses, and conclusions, and remain fully responsible for the accuracy and integrity of the final manuscript.

\bibliography{references}
\bibliographystyle{iclr2026_conference}

\appendix
\section{Theoretical Analysis of Candidate Competition}
\label{app:candidate_competition_proof}

To provide a theoretical explanation for the trajectory evaluation bottleneck
discussed in Sec.~\ref{sec:trajectory_evaluation_bottleneck}, we consider a
stylized setting in which an oracle trajectory competes with multiple
near-optimal candidates under imperfect evaluation. The analysis characterizes
how increasing candidate competition affects oracle recovery when evaluation
errors are present.

\begin{proposition}[Candidate Competition under Imperfect Evaluation]
\label{thm:candidate_competition}
Let $r_i=\mathcal{R}(\bm{\tau}_i)$ denote the reference planning score of
candidate $\bm{\tau}_i$. Consider a candidate set containing one oracle
trajectory $i^\star$ with score $r^\star$ and $K-1$ competing trajectories
with score $r^\star-\gamma$, where $\gamma>0$. Suppose the evaluator predicts
\begin{equation}
s_i=r_i+\epsilon_i,
\end{equation}
where $\{\epsilon_i\}_{i=1}^{K}$ are i.i.d. evaluation errors drawn from a
continuous distribution with cumulative distribution function $F$.

Let
$\hat{i}=\arg\max_i s_i$
denote the selected candidate and define the evaluation gap as
$\Delta_{\mathrm{eval}}(K)=r^\star-r_{\hat{i}}$.
Then the probability of recovering the oracle trajectory is
\begin{equation}
P_K
=
\Pr(\hat{i}=i^\star)
=
\mathbb{E}_{\epsilon^\star}
\left[
F(\epsilon^\star+\gamma)^{K-1}
\right],
\label{eq:selection_probability}
\end{equation}
where $\epsilon^\star$ is the evaluation error associated with the oracle
trajectory. The expected evaluation gap is therefore
\begin{equation}
\mathbb{E}
\left[
\Delta_{\mathrm{eval}}(K)
\right]
=
\gamma(1-P_K).
\label{eq:expected_evaluation_gap}
\end{equation}

Moreover, if
$\Pr\!\left(0<F(\epsilon^\star+\gamma)<1\right)>0$,
then $P_K$ strictly decreases with $K$, while
$\mathbb{E}[\Delta_{\mathrm{eval}}(K)]$ strictly increases with $K$.
\end{proposition}

\begin{proof}
Condition on the oracle evaluation error $\epsilon^\star=e$.
A competing trajectory is ranked below the oracle trajectory if
\[
r^\star-\gamma+\epsilon_i
<
r^\star+e,
\]
which is equivalent to
$\epsilon_i<e+\gamma$.
This event occurs with probability $F(e+\gamma)$.
Since the $K-1$ competing trajectories have independent evaluation errors,
the conditional probability that all of them are ranked below the oracle is
\begin{equation}
\Pr(\hat{i}=i^\star\mid\epsilon^\star=e)
=
F(e+\gamma)^{K-1}.
\end{equation}
Taking the expectation over $\epsilon^\star$ yields
\eqref{eq:selection_probability}.

Under the assumed score structure, the evaluation gap is zero when the oracle
trajectory is selected and equals $\gamma$ otherwise. Hence,
\begin{equation}
\mathbb{E}
\left[
\Delta_{\mathrm{eval}}(K)
\right]
=
\gamma\Pr(\hat{i}\neq i^\star)
=
\gamma(1-P_K),
\end{equation}
which gives \eqref{eq:expected_evaluation_gap}.

To establish monotonicity, define
\[
X=F(\epsilon^\star+\gamma)\in[0,1].
\]
Then
\begin{equation}
P_K-P_{K+1}
=
\mathbb{E}
\left[
X^{K-1}(1-X)
\right].
\end{equation}
Since $X^{K-1}(1-X)\geq0$, we have $P_{K+1}\leq P_K$.
Furthermore, if
$\Pr(0<X<1)>0$,
then
$\mathbb{E}[X^{K-1}(1-X)]>0$,
and therefore
$P_{K+1}<P_K$.
Combining this result with
\eqref{eq:expected_evaluation_gap} immediately gives
\[
\mathbb{E}[\Delta_{\mathrm{eval}}(K+1)]
>
\mathbb{E}[\Delta_{\mathrm{eval}}(K)].
\]
Thus, under imperfect evaluation, increasing the number of competing
trajectories reduces the probability of recovering the oracle and enlarges
the expected evaluation gap.
\end{proof}

This result provides a simplified explanation for the
generation--evaluation asymmetry observed in multi-candidate planning.
Increasing the number of trajectory hypotheses can improve candidate coverage
and oracle potential, but it also exposes the evaluator to more competing
alternatives. Consequently, when evaluation remains imperfect, stronger
candidate generation alone does not guarantee a comparable improvement in final trajectory selection.

\section{Interpretation of Candidate Space Refinement}
\label{app:rft_derivation}

We provide a simple interpretation of the exponential weighting used in
Stage~2. Let
$\bm{\xi}=(\bm{x},\bm{\tau}^{\star})$
denote a supervised training example drawn from the empirical demonstration
distribution $p_{\mathcal D}(\bm{\xi})$, and let
$R(\bm{\xi})=R(\bm{x},\bm{\tau}^{\star})$ denote its downstream planning reward.
We consider a reward-tilted demonstration distribution that favors
higher-reward examples while remaining close to the original training
distribution:
\begin{equation}
q^{\star}
=
\arg\max_q
\left[
\mathbb{E}_{\bm{\xi}\sim q}[R(\bm{\xi})]
-
\frac{1}{\beta}
D_{\mathrm{KL}}
\left(
q(\bm{\xi})\Vert p_{\mathcal D}(\bm{\xi})
\right)
\right].
\label{eq:app_reward_tilt}
\end{equation}

Introducing a Lagrange multiplier for the normalization constraint
$\int q(\bm{\xi})\,d\bm{\xi}=1$, the stationary condition is
\begin{equation}
R(\bm{\xi})
-
\frac{1}{\beta}
\left(
\log
\frac{q^{\star}(\bm{\xi})}
{p_{\mathcal D}(\bm{\xi})}
+1
\right)
+\lambda
=0,
\end{equation}
which yields the Gibbs-form solution
\begin{equation}
q^{\star}(\bm{\xi})
=
\frac{
p_{\mathcal D}(\bm{\xi})
\exp\left(\beta R(\bm{\xi})\right)
}{
Z_{\beta}
},
\qquad
Z_{\beta}
=
\mathbb{E}_{\bm{\xi}\sim p_{\mathcal D}}
\left[
\exp\left(\beta R(\bm{\xi})\right)
\right].
\label{eq:app_gibbs_solution}
\end{equation}
Therefore, Eq.~\eqref{eq:app_reward_tilt} corresponds to exponentially
tilting the original demonstration distribution according to downstream
planning quality.

Training the candidate generator under the tilted distribution in
Eq.~\eqref{eq:app_gibbs_solution} gives
\begin{equation}
\mathbb{E}_{\bm{\xi}\sim q^{\star}}
\left[
\mathcal{L}_{\mathrm{traj}}(\bm{\xi};\phi)
\right]
=
\frac{1}{Z_{\beta}}
\mathbb{E}_{\bm{\xi}\sim p_{\mathcal D}}
\left[
\exp\left(\beta R(\bm{\xi})\right)
\mathcal{L}_{\mathrm{traj}}(\bm{\xi};\phi)
\right].
\label{eq:app_weighted_rft}
\end{equation}
Since $Z_{\beta}$ is independent of the candidate-generator parameters
$\phi$, minimizing Eq.~\eqref{eq:app_weighted_rft} is equivalent, up to a
positive constant factor, to the reward-weighted objective in
Eq.~\eqref{eq:rft_objective}.

Importantly, our supervision is deterministic: each scene $\bm{x}$ is
associated with a single expert trajectory
$\bm{\tau}^{\star}=f^{\star}(\bm{x})$, such that
\begin{equation}
p_{\mathcal D}(\bm{x},\bm{\tau})
=
p_{\mathcal D}(\bm{x})
\delta\left(
\bm{\tau}-f^{\star}(\bm{x})
\right).
\end{equation}
Under this setting, the exponential tilting in
Eq.~\eqref{eq:app_gibbs_solution} does not change the supervision target
within an individual scene. Instead, it reweights the relative contributions
of different demonstrations, assigning greater optimization weight to scenes
whose expert trajectories achieve higher downstream planning rewards.
Consequently, Stage~2 remains a fully supervised refinement procedure, while
the shared candidate generator is preferentially optimized toward trajectory
patterns associated with higher planning quality.

\section{Experiment Details}
\subsection{Evaluation Metrics}
\label{app:evaluation_metrics}

We employ benchmark-specific metrics to comprehensively evaluate planning
performance. For closed-loop planning, we report the Predictive Driver Model
Score (PDMS) on NAVSIM v1 and the Extended Predictive Driver Model Score
(EPDMS) on NAVSIM v2.

\paragraph{PDMS on NAVSIM v1.}
PDMS integrates five sub-metrics: No At-Fault Collision (NC), Drivable Area
Compliance (DAC), Time-to-Collision (TTC), Comfort (C), and Ego Progress (EP).
It is computed as
\begin{equation}
\mathrm{PDMS}
=
\mathrm{NC}
\times
\mathrm{DAC}
\times
\left(
\frac{
5 \times \mathrm{EP}
+
5 \times \mathrm{TTC}
+
2 \times \mathrm{C}
}{12}
\right).
\label{eq:pdms}
\end{equation}

\paragraph{EPDMS on NAVSIM v2.}
EPDMS extends PDMS with additional components that provide a more comprehensive
evaluation of safety, traffic-rule compliance, driving progress, lane keeping,
and comfort. It includes No At-Fault Collision (NC), Drivable Area Compliance
(DAC), Driving Direction Compliance (DDC), Traffic Light Compliance (TLC),
Ego Progress (EP), Lane Keeping (LK), History Comfort (HC),
Time-to-Collision (TTC), and Extended Comfort (EC). It is computed as
\begin{equation}
\mathrm{EPDMS}
=
\mathrm{NC}
\times
\mathrm{DAC}
\times
\mathrm{DDC}
\times
\mathrm{TLC}
\times
\left(
\frac{
5 \times \mathrm{EP}
+
2 \times \mathrm{LK}
+
2 \times \mathrm{HC}
+
5 \times \mathrm{TTC}
+
2 \times \mathrm{EC}
}{16}
\right).
\label{eq:epdms}
\end{equation}





\subsection{Implementation Details}
\label{app:Implementation details}

The trajectory generator follows the standard configuration of DrivoR~\citep{kirby2026drivor}. It takes 4 camera views and the ego state as input. We use DINOv2 ViT-S as the visual encoder and adapt it using LoRA, producing 16 visual registers for each view. The trajectory decoder and Scorer each consist of 4 layers. The Safety Head is implemented as a 2-layer MLP that maps candidate features to 2 binary risk predictions: collision and departure from the drivable area. For semantic modulation, we use a frozen InternVL3-2B initialized with the pretrained weights from ReCogDrive. Its hidden features condition an 8-block DiT Modulator. The main model and training configurations are summarized in Table~\ref{tab:model_details}.
\begin{table}[t]
\centering
\caption{\textbf{Model and training configurations.}}
\label{tab:model_details}

\small
\setlength{\tabcolsep}{12pt}
\renewcommand{\arraystretch}{0.98}

\begin{tabular}{@{}lc@{}}
\toprule
\textbf{Configuration} & \textbf{Value} \\
\midrule

Candidate number $K$
& 64 \\

Future planning steps $T$
& 8 \\

Planning time interval
& 0.5 s \\

Visual registers per view
& 16 \\

Trajectory decoder layers
& 4 \\

Scorer layers
& 4 \\

DiT blocks
& 8 \\

DiT sampling steps
& 5 \\

Global batch size
& 64 \\

\midrule

Safety threshold $\delta$
& 0.95 \\

Evaluator loss weight $\lambda_{\mathrm{eval}}$
& 1 \\

Reward coefficient $\beta$
& 1 \\

Profiles per group $G$
& 8 \\

Top-$k$ candidates $k$
& 16 \\

Lower clipping bound $\epsilon_{\mathrm{low}}$
& 0.2 \\

Upper clipping bound $\epsilon_{\mathrm{high}}$
& 0.28\\

\bottomrule
\end{tabular}
\end{table}
\section{Supplementary Experimental Results}
\subsection{Computational Time and Memory Analysis}
\label{app:time_memory}
We profile the computational and memory overhead of iDriveVLA on an A800 GPU with a single-element batch. The complete pipeline requires 297.20 ms per sample, only 186.10 ms more than DrivoR. This sub-second latency satisfies the real-time inference requirements of practical driving applications.
\begin{table}[t]
    \centering
    \small
\caption{Computational and memory overhead of the iDriveVLA inference pipeline.}
    \label{tab:runtime_memory}
    \renewcommand{\arraystretch}{1.15}
    \begin{tabular}{lcc}
        \toprule
        \textbf{Module}
        & \textbf{Time (ms)}
        & \textbf{Parameter memory (MiB)} \\
        \midrule

        Vision Encoder
        & 97.73
        & 90.53 \\

        Action Decoder
        & 1.25
        & 41.74 \\

        Scorer
        & 3.40
        & 7.15 \\

        Safety Head
        & 0.20
        & 2.02 \\

        Modulator
        & 194.62
        & 3999.51 \\

        \midrule

        \textbf{Total}
        & \textbf{297.20}
        & \textbf{4140.95} \\

        \bottomrule
    \end{tabular}
\end{table}

\subsection{Qualitative Analysis of Trajectory Selection}
\label{app:oracle_selection}

\paragraph{Selection failures of DrivoR.}
\begin{figure*}[!t]
  \centering
  \begin{minipage}[t]{0.495\textwidth}
    \centering
    \includegraphics[width=\linewidth]{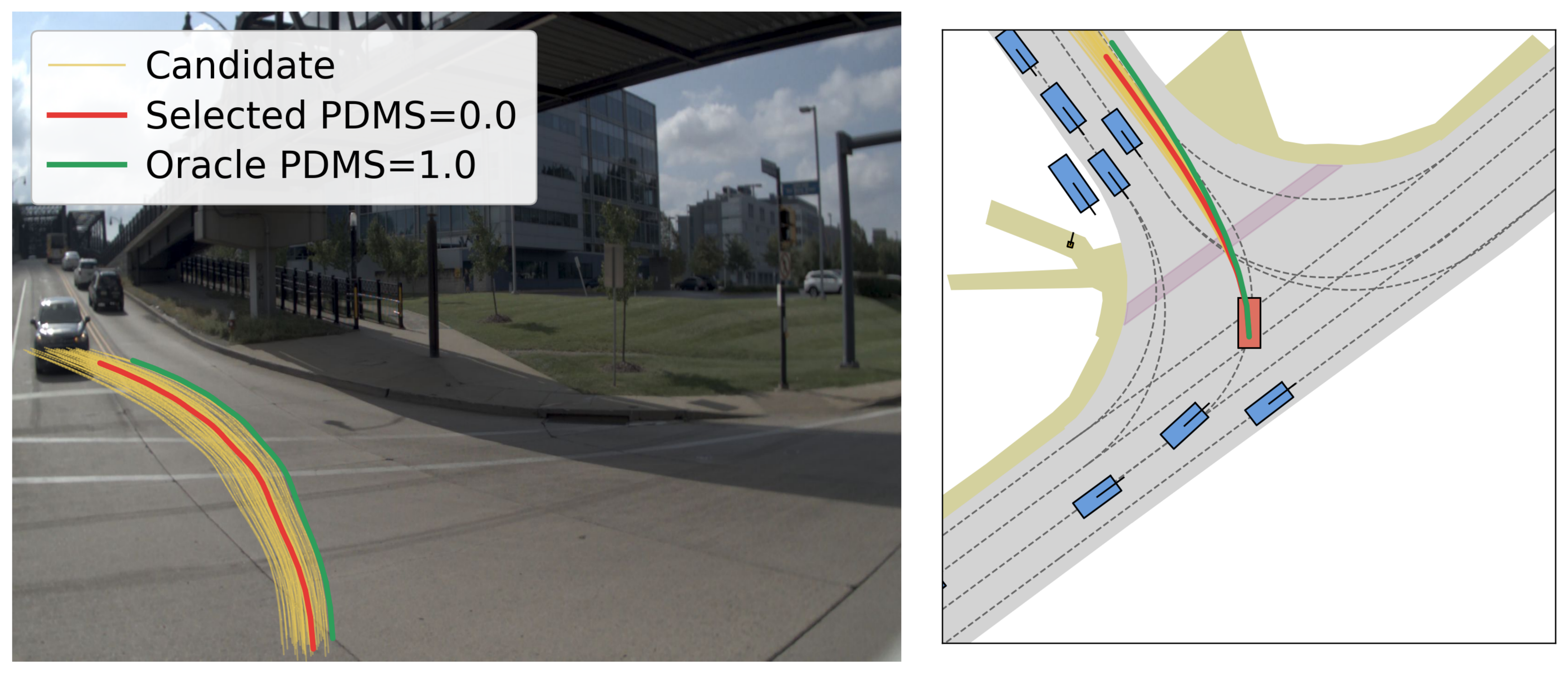}
  \end{minipage}\hfill
  \begin{minipage}[t]{0.495\textwidth}
    \centering
    \includegraphics[width=\linewidth]{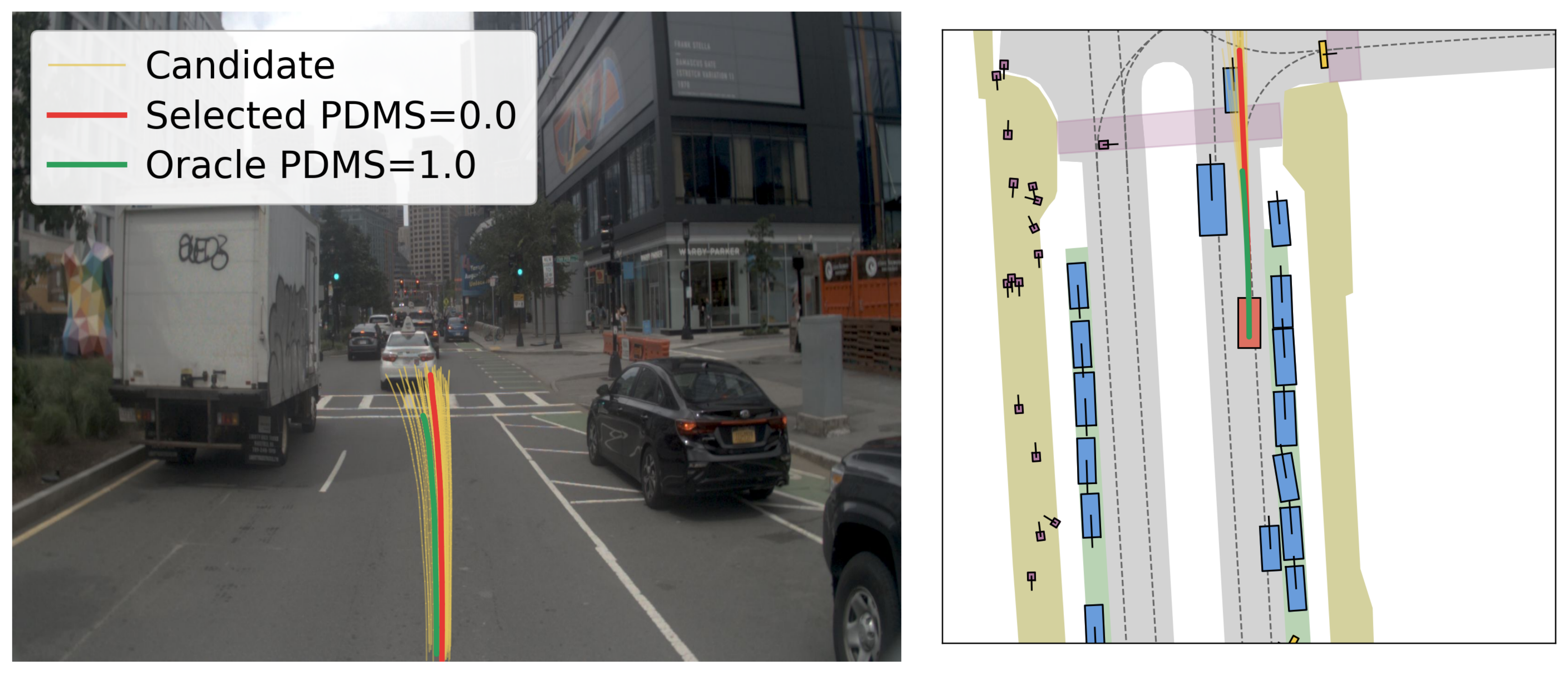}
  \end{minipage}

  \vspace{2mm}

  \begin{minipage}[t]{0.495\textwidth}
    \centering
    \includegraphics[width=\linewidth]{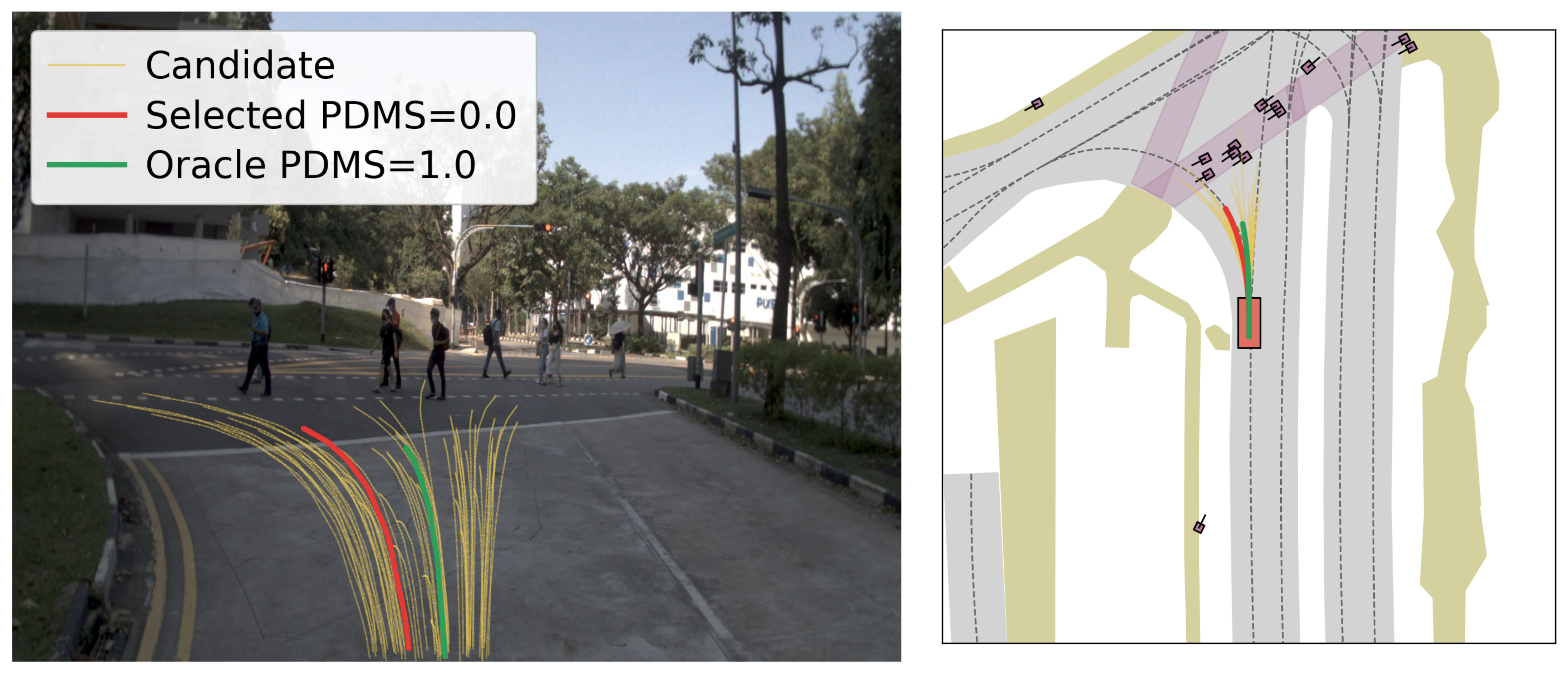}
  \end{minipage}\hfill
  \begin{minipage}[t]{0.495\textwidth}
    \centering
    \includegraphics[width=\linewidth]{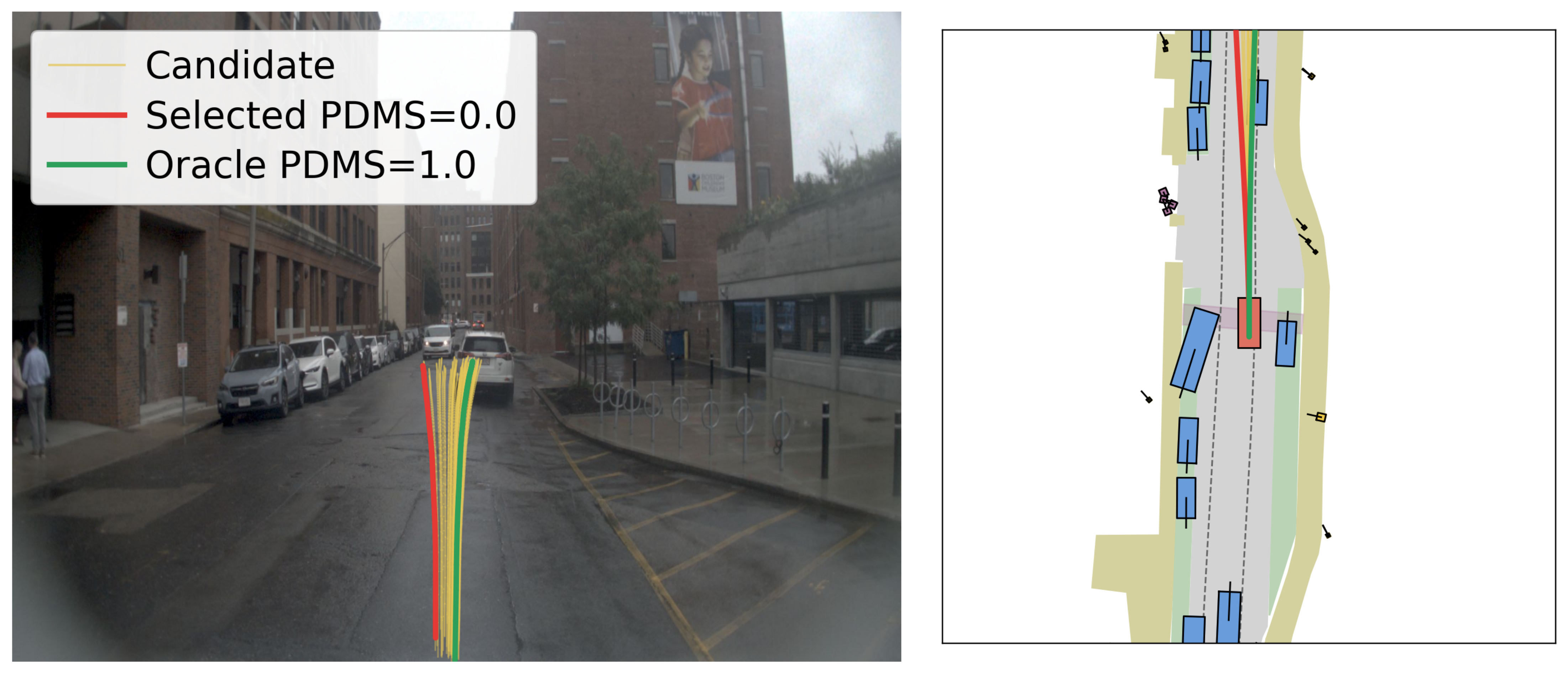}
  \end{minipage}

\caption{
\textbf{Representative catastrophic selection failures of DrivoR.}
Yellow curves denote the candidate trajectories, while red and green curves
denote the selected and oracle trajectories, respectively. Despite the
presence of a safe oracle trajectory, the evaluator selects an unsafe
candidate with zero PDMS.
}
  \label{fig:drivor_failure_cases}
\end{figure*}
Figure~\ref{fig:drivor_failure_cases} shows representative failures at
intersections and in interactions with vehicles and pedestrians. The
candidate overlays demonstrate that the generator already provides multiple
plausible alternatives, including a safe oracle trajectory. Nevertheless,
DrivoR selects an unsafe candidate with zero PDMS. Although the selected and
oracle trajectories differ only subtly in curvature or lateral offset, these
differences determine collision risk and road compliance, exposing the
evaluator's limited sensitivity to fine-grained safety risks and interaction
semantics.

\paragraph{Selection results of iDriveVLA.}
\begin{figure*}[!t]
  \centering
  \begin{minipage}[t]{0.495\textwidth}
    \centering
    \includegraphics[width=\linewidth]{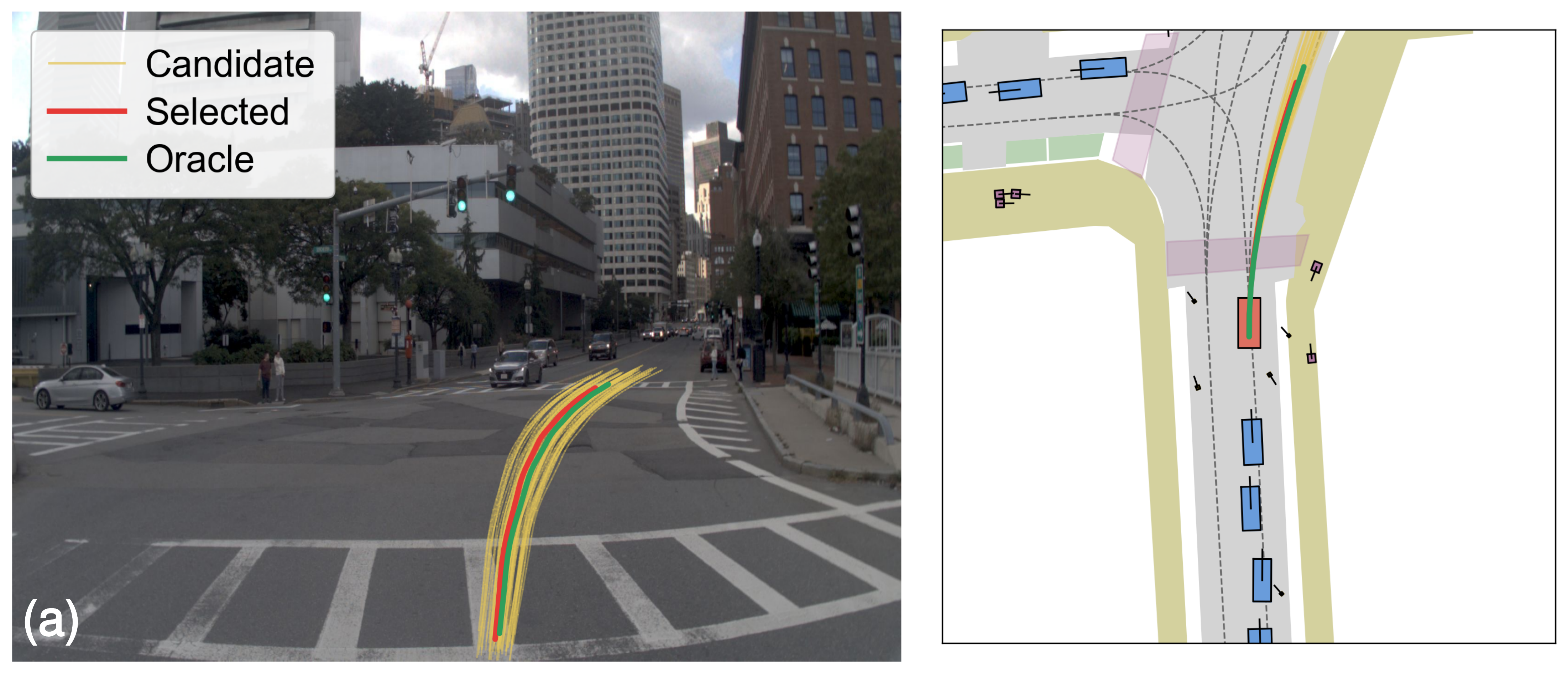}
  \end{minipage}\hfill
  \begin{minipage}[t]{0.495\textwidth}
    \centering
    \includegraphics[width=\linewidth]{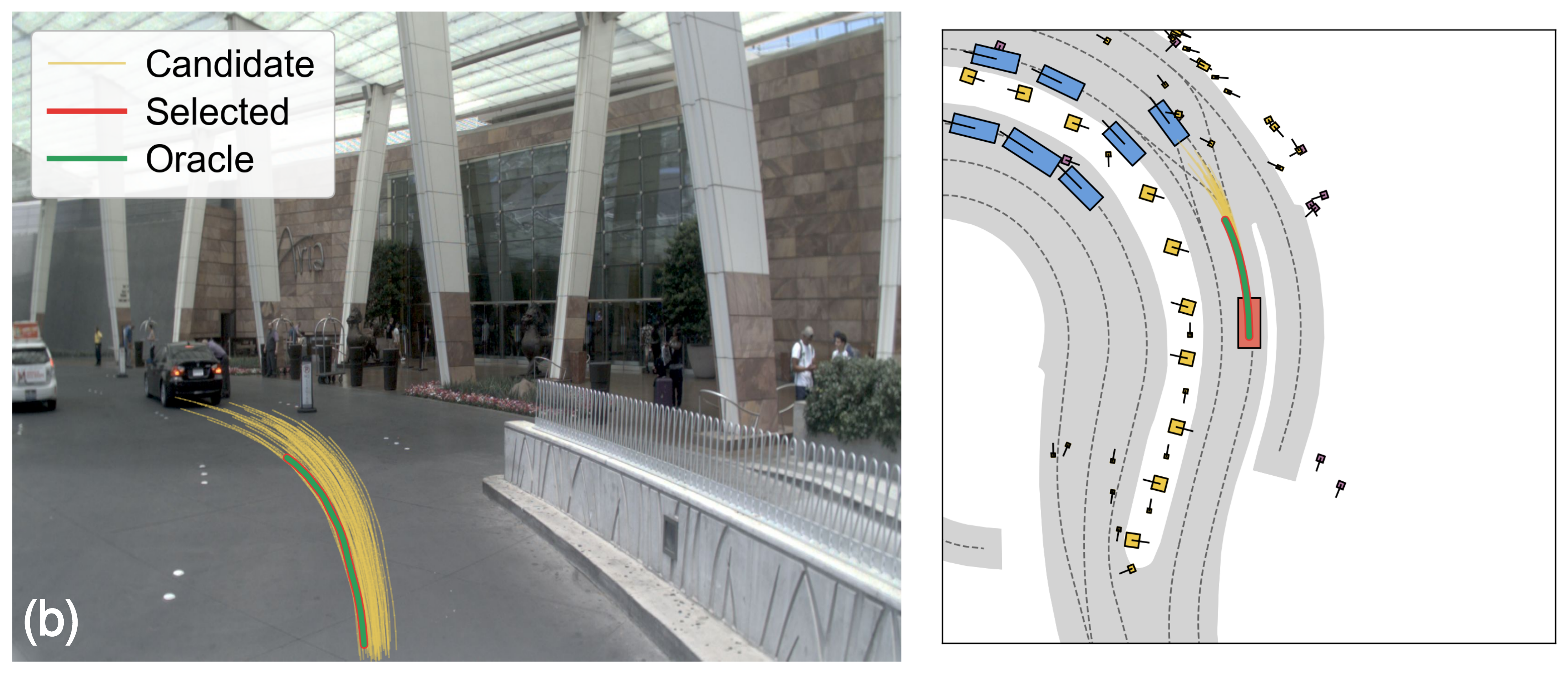}
  \end{minipage}
  \caption{\textbf{Qualitative comparison between the trajectory selected by
iDriveVLA and the oracle trajectory.} The selected trajectories closely follow or overlap with the oracle trajectories, demonstrating reliable identification of high-quality candidates.}
  \label{fig:qualitative_oracle}
\end{figure*}
We further compare the trajectories selected by iDriveVLA with the oracle
trajectories in Figure~\ref{fig:qualitative_oracle}. The oracle trajectory is
defined as the candidate achieving the highest ground-truth PDMS within the
generated candidate set. In
Figure\hyperref[fig:qualitative_oracle]{~\ref*{fig:qualitative_oracle}(a)},
the selected trajectory closely follows the oracle trajectory, with only a
minor geometric deviation, demonstrating a near-oracle selection. In
Figure\hyperref[fig:qualitative_oracle]{~\ref*{fig:qualitative_oracle}(b)},
iDriveVLA successfully identifies the oracle candidate, as indicated by the
overlap between the selected and oracle trajectories. These examples show
that the proposed evaluator can distinguish high-quality candidates from a
diverse trajectory set and select optimal or near-optimal driving behaviors.

\end{document}